\documentclass[twoside]{article}

\usepackage[accepted]{aistats2026}
\usepackage[round]{natbib}

\usepackage[table, svgnames, dvipsnames]{xcolor}

\usepackage{amssymb}
\usepackage{mathtools}
\usepackage{amsthm}
\usepackage{graphicx}
\usepackage{subfigure}
\usepackage{booktabs}
\usepackage{algorithm}
\usepackage[noend]{algpseudocode}
\usepackage{bbm}
\usepackage{url}
\usepackage{hyperref}

\usepackage{tikz}
\usepackage{graphicx}
\usepackage{amsmath}

\usetikzlibrary{shapes, arrows.meta, positioning, calc}

\newtheorem{theorem}{Theorem}[section]

\newtheorem{lemma}[theorem]{Lemma}
\newtheorem{corollary}[theorem]{Corollary}

\newtheorem{remark}[theorem]{Remark}

\newtheorem{innercustomthm}{Theorem}
\newenvironment{customthm}[1]
  {\renewcommand\theinnercustomthm{#1}\innercustomthm}
  {\endinnercustomthm}
  
\newtheorem{innercustomlemma}{Lemma}
\newenvironment{customlemma}[1]
  {\renewcommand\theinnercustomlemma{#1}\innercustomlemma}
  {\endinnercustomlemma}

\DeclareMathOperator*{\argmin}{arg\,min}

\DeclareMathOperator*{\rk}{rank}
\DeclareMathOperator*{\tr}{trace}

\begin{document}

%

%

\twocolumn[

\aistatstitle{Convexified Message-Passing Graph Neural Networks}

\aistatsauthor{ Saar Cohen \And Noa Agmon \And  Uri Shaham }

\aistatsaddress{ Department of Computer Science \\ Bar-Ilan University, Israel \And Department of Computer Science \\ Bar-Ilan University, Israel \And Department of Computer Science \\ Bar-Ilan University, Israel } ]

\begin{abstract}
  Graph Neural Networks (GNNs) are key tools for graph representation learning, demonstrating strong results across diverse prediction tasks. In this paper, we present \textit{\textbf{Convexified Message-Passing Graph Neural Networks}} (CGNNs), a novel and general framework that combines the power of message-passing GNNs with the tractability of \textit{convex} optimization. By mapping their nonlinear filters into a reproducing kernel Hilbert space, CGNNs transform training into a convex optimization problem, which projected gradient methods can solve both efficiently and optimally. Convexity further allows CGNNs' statistical properties to be analyzed accurately and rigorously. For two-layer CGNNs, we establish rigorous generalization guarantees, showing convergence to the performance of an optimal GNN. To scale to deeper architectures, we adopt a principled layer-wise training strategy. Experiments on benchmark datasets show that CGNNs significantly exceed the performance of leading GNN models, obtaining 10–40\% higher accuracy in most cases, underscoring their promise as a powerful and principled method with strong theoretical foundations. In rare cases where improvements are not quantitatively substantial, the convex models either slightly exceed or match the baselines, stressing their robustness and wide applicability. Though over-parameterization is often used to enhance performance in non-convex models, we show that our CGNNs yield shallow convex models that can surpass non-convex ones in accuracy and model compactness.
\end{abstract}

\section{INTRODUCTION}
\label{sec:intro}
Graphs serve as versatile models across various domains, such as social networks \citep{wang2018deep}, protein design \citep{ingraham2019generative}, and medical diagnosis \citep{li2020graphdiagnosis}. Graph representation learning has thus attracted widespread attention in recent years, with Graph Neural Networks (GNNs) emerging as powerful tools. Message Passing Neural Networks (MPNNs) \citep{gilmer2017neural,hamilton2017inductive,velivckovic2018graph}, a most widely adopted class of GNNs, provide a scalable and elegant architecture that effectively exploits a message-passing mechanism to extract useful information and learn high-quality representations from graph data.

In this paper, we present the new and general framework of \textbf{\textit{Convexified Message-Passing Graph Neural Networks}} (CGNNs), reducing the training of message-passing GNNs to a \textit{convex} optimization problem solvable efficiently and optimally via projected gradient methods. The convexity of CGNNs offers a distinct advantage: it ensures \textit{global} optimality and computational tractability without relying on, e.g., over-parameterization, unlike many existing approaches \citep{allen2019convergence,du2018power,du2018gradient,zou2020gradient}. This convexity also enables a precise and theoretically grounded analysis of generalization and statistical behavior. 

We exhibit our framework by convexifying two-layer Graph Convolution Networks (GCNs) \citep{defferrard2016convolutional,kipf2017semisupervised}, a widely-used GNN model for learning over graph-structured data. Inspired by the convexification of convolutional neural networks by \citet{zhang2017convexified}, we adapt their methodology to the graph domain. A high-level overview of our convexification pipeline is summarized in Figure~\ref{fig:cgnn-schematic}, explained as follows. The key challenge lies in handling GCNs' nonlinearity: we address this by first relaxing the space of GCN filters to a reproducing kernel Hilbert space (RKHS) (step (1) in Figure~\ref{fig:cgnn-schematic}), effectively reducing the model to one with a linear activation function (step (2)). We then show that GCNs with RKHS filters admit a low-rank matrix representation (step (3)). Finally, by relaxing the non-convex rank constraint into a convex nuclear norm constraint (step (4)), we arrive at our CGCN formulation.

Theoretically, we establish an upper bound on the expected generalization error of our class of two-layer CGCNs, composed of the best possible performance attainable by a two-layer GCN with infinite data and a complexity term that decays to zero polynomially with the number of training samples. Our result dictates that the generalization error of a CGCN provably converges to that of an optimal GCN, offering both strong statistical guarantees and practical relevance. Moreover, our analysis naturally extends to deeper models by applying the argument layer-wise, providing insights into the statistical behavior of multi-layer CGCNs as well.
\begin{figure*}[t!]
\centering
\begin{tikzpicture}[node distance=1.2cm and 0.8cm]

\tikzset{
  mybox/.style={
    rectangle, draw=black, minimum width=2.2cm, minimum height=1.3cm,
    align=center, fill=blue!15
  },
  blackbox/.style={
    mybox, fill=ForestGreen, text=white
  },
  graybox/.style={
    mybox, fill=gray!20
  },
  subbox/.style={
    rectangle, draw=black,
    minimum width=2.4cm, minimum height=0.7cm, align=center
  },
  arrow/.style={
    thick, ->, >=stealth,
    shorten >=1pt, shorten <=1pt
  }
}

\node (a) [mybox, minimum width=1cm] {GNN};

\node (b) [mybox, right=of a, minimum width=1cm] {\textbf{(1)}\\Map to\\RKHS};

\node (c) [mybox, right=of b, minimum width=1cm] {\textbf{(2)}\\Linear\\activation};

\node (d) [mybox, right=of c] {\textbf{(3)}\\Low-rank matrix\\parameterization};

\node (e) [mybox, right=of d] {\textbf{(4)}\\Nuclear norm\\relaxation};

\node (f) [blackbox, right=of e, minimum width=1cm] {\textbf{Train}\\\textbf{CGNN}};

\draw[arrow] (a) -- (b);
\draw[arrow] (b) -- (c);
\draw[arrow] (c) -- (d);
\draw[arrow] (d) -- (e);
\draw[arrow] (e) -- (f);

\end{tikzpicture}

\caption{Convexification Procedure of a message-passing GNN into a CGNN.}
\label{fig:cgnn-schematic}
\end{figure*}
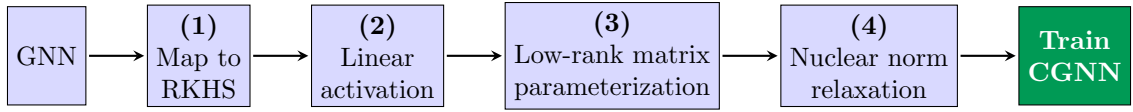
Finally, we perform an extensive empirical analysis over various benchmark datasets, where our framework exhibits outstanding results. First, we apply our framework to a wide range of leading message-passing GNN models, including GAT \citep{velivckovic2018graph}, GATv2 \citep{brody2022how}, GIN \citep{xu2018how}, GraphSAGE \citep{hamilton2017inductive}, and ResGatedGCN \citep{bresson2017residual}. We also extended our framework to hybrid transformer models, specifically GraphGPS \citep{rampavsek2022recipe} and GraphTrans \citep{shi2021masked,wu2021representing}, which integrate local message passing GNNs with global attention mechanisms. Remarkably, in most cases, the convex variants outperform the original non-convex counterparts by a large magnitude, obtaining 10–40\% higher accuracy, highlighting the practical power of our approach alongside its theoretical appeal. Even in the minority of cases where improvements are modest, the convexified models still match or slightly exceed the baseline performance, further validating the robustness and broad applicability of our method. While over-parameterization is commonly used to boost performance in non-convex models \citep{allen2019convergence,du2018power,du2018gradient,zou2020gradient}, we show that our CGNNs framework enables the construction of shallow convex architectures that outperform over-parameterized non-convex ones in both accuracy and model compactness.

\section{RELATED WORK}
\label{sec:related-work}
\textbf{Message-Passing GNNs.} Early developments of message-passing GNNs include GCN \citep{defferrard2016convolutional,kipf2017semisupervised}, GAT \citep{velivckovic2018graph}, GIN \citep{xu2018how}, and GraphSAGE \citep{hamilton2017inductive}. Such methods effectively exploit a message-passing mechanism to recursively aggregate useful neighbor information \citep{gilmer2017neural}, enabling them to learn high-quality representations from graph data. Those models have proven effective across diverse tasks, obtaining state-of-the-art performance in various graph prediction tasks (see, e.g., \citep{chen2022structure,hamilton2017inductive,rampavsek2022recipe,velivckovic2018graph,xu2018how,you2021graph}). Our CGNNs harness the power of message-passing GNNs by transforming training into a \textit{convex} optimization problem. This shift not only enables efficient computation of globally optimal solutions, but also allows an accurate and rigorous analysis of statistical properties. In Appendix \ref{supp:ADDITIONAL RELATED WORK}, we provide more details about hybrid transformer models.

\textbf{Convergence of Gradient Methods to Global Optima.} A vast line of research analyzed the convergence of gradient-based methods to global optima when training neural networks, but under restrictive conditions. 
Several works rely on over-parameterization, often requiring network widths far exceeding the number of training samples \citep{allen2019convergence,du2018gradient,zou2020gradient}, with milder assumptions in shallow networks \citep{du2018power,soltanolkotabi2018theoretical}. However, the convexity of our CGNNs enables global optimality and computational efficiency without relying on any of the above restrictive assumptions.

\textbf{Convex Neural Networks.} The area of convex neural networks is relatively underexplored. \citet{zhang2016regularized} proposed a convex relaxation step for fully-connected neural networks, which inspired the convexification of convolutional neural networks (CNNs) \citep{zhang2017convexified}; see Appendix \ref{supp:ccnns} for an extensive comparison to \citep{zhang2017convexified}. Other works study linear CNNs \citep{gunasekar2018implicit} and linear GNNs \citep{shin2023efficient}, which are of less interest in practice since they collapse all layers into one. Recently, \citet{cohen2021convex} proposed convex versions of specific GNN models termed as aggregation GNN, which apply only to a signal with temporal structure \citep{gama2018convolutional}. Their work is narrowly tailored and focused on a specific application in swarm robotics \citep{cohen2021spatial,cohen2021recent,cohen2020converging}, where convexity arises from modeling assumptions tailored to domain adaptation. In contrast, our CGNN framework establishes a general and theoretically principled convexification of diverse message-passing GNNs by transforming the standard architecture itself, independent of application-specific constraints. Unlike their formulation that does not preserve, e.g., standard GNN modularity, our approach retains these key structural features and applies broadly to a wide range of graph learning tasks. In particular, our contribution resolves the key limitations of \cite{cohen2021convex} through three fundamental distinctions: (1) \textbf{Scope}, by enabling training on variable-sized graphs (inductive setting) versus their fixed-topology restriction; (2) \textbf{Theory}, by extending analysis to general non-polynomial filters and proving strictly tighter nuclear norm bounds; and (3) \textbf{Methodology}, by utilizing an RKHS formulation that supports modern optimizers and generic modular tasks. In Appendix \ref{supp:compare to myself}, we supply an extensive comparison to \cite{cohen2021convex}, detailing the fundamental distinctions from their work. As such, CGNNs represent a scalable and domain-agnostic foundation for convex GNN training, grounded in a fundamentally different mathematical framework.

\textbf{Classical Graph Kernels.} Our framework offers a unique theoretical bridge between modern deep graph learning and classical graph kernels. While GNNs have largely superseded kernel methods in popularity, establishing the formal connections between them contextualizes the theoretical advances of the CGNN framework, particularly regarding the role of Reproducing Kernel Hilbert Spaces (RKHS) and optimization landscapes. We supply a detailed comparison to classic graph kernels in Appendix \ref{supp:graph kernels}.

\textbf{GNNs as Unified Optimization Problems.} Our framework relates to and fundamentally differs from prior work on GNNs as unified optimization problems. Specifically, \cite{chen2023bridging} classify spatial GNNs based on how they aggregate neighbor information, while discussing how graph convolution can be depicted as an optimization problem. They also discuss how filters of spectral GNNs can be approximated via either linear, polynomial or rational approximation. The works by \cite{he2021bernnet}, \cite{wang2022powerful} and \cite{guo2023graph} focus on polynomial approximation, treating spectral filters as parameterized by a polynomial basis (see Appendix \ref{supp:GNNs as Unified Optimization Problems} for details). However, in all those works training remains \textit{\textbf{non-convex}}. Our contribution is thus orthogonal and complementary: rather than deriving architectures, we \textbf{\textit{convexify the training problem}} itself for general, non-convex message-passing GNNs (not only spectral ones). Training thus becomes a convex optimization problem, ensuring global optimality as well as an accurate and rigorous analysis of statistical properties. Further, our framework applies directly to general message-passing GNNs and hybrid transformer architectures like GraphGPS and GraphTrans, whereas The works by \cite{he2021bernnet}, \cite{wang2022powerful} and \cite{guo2023graph} only cover spectral GNNs. In fact, the optimization-oriented analysis of \cite{wang2022powerful} is restricted to \textit{linear} GNNs, while we focus on convexifying \textbf{\textit{non-convex}} GNNs. We will add a dedicated discussion explicitly contrasting these settings.

\textbf{Notations.} For brevity, $[n]$ denotes the discrete set $\{1, 2,\dots , n\}$ for any $n \in \mathbb{N}_{>0}$. 
For a rectangular matrix $A$, let $\|A\|_*$ be its nuclear norm (i.e., the sum of $\mathcal{A}_{\ell}$'s singular values), and $\|A\|_2$ be its spectral norm (i.e., maximal singular value). Let $\ell^2(\mathbb{N})$ be the set of countable dimensional of square-summable sequences $v$, i.e., $\sum_{i=1}^\infty v_i^2 < \infty$. 

\section{PRELIMINARIES}
\label{sec:gnn}
Network data can be captured by an underlying undirected graph $\mathcal{G} = (\mathcal{V},\mathcal{E}, \mathcal{W})$, where $\mathcal{V} = \{1,...,n\}$ denotes the vertex set, $\mathcal{E}$ denotes the edge set and $\mathcal{W}:\mathcal{E} \rightarrow \mathbb{R}$ denotes the edge weight function, while a missing edge $(i,j)$ is depicted by $\mathcal{W}(i,j) = 0$. The {\it neighborhood} of node $i \in \mathcal{V}$ is denoted by $N_i:= \{j \in \mathcal{V} | (j,i) \in \mathcal{E}\}$. The data on top of this graph is modeled as a \textit{graph signal} $\mathcal{X} \in \mathbb{R}^{n \times F}$, whose $i^{\text{th}}$ row $[\mathcal{X}]_{i}=\mathbf{x}_i$ is an $F$-dimensional \textit{feature vector} assigned to node $i \in \mathcal{V}$. We leverage the framework of {\it graph signal processing} (GSP) \citep{ortega:graph-signal-processing} as the mathematical foundation for learning from graph signals. For relating the graph signal $\mathcal{X}$ with the underlying structure of $\mathcal{G}$, we define a {\it graph shift operator} (GSO) by $\mathcal{S} \in \mathbb{R}^{n \times n}$, which 
satisfies $[\mathcal{S}]_{ij} = s_{ij} = 0$ if $(j,i) \notin \mathcal{E}$ for $j \neq i$. Commonly utilized GSOs include the \textbf{adjacency} matrix \citep{sandryhaila2013discrete,sandryhaila2014discrete}, the \textbf{Laplacian} matrix \citep{shuman2013emerging}, and their normalized counterparts \citep{defferrard2016convolutional,gama2018diffusion}. By viewing $\mathcal{S}$ as a linear operator, the graph signal $\mathcal{X}$ can be {\it shifted} over the nodes, yielding that the output at node $i$ for the $f^{\text{th}}$ feature becomes $[\mathcal{S}\mathcal{X}]_{if} = \sum_{j=1}^n [\mathcal{S}]_{ij}[\mathcal{X}]_{jf} = \sum_{j \in N_i} s_{ij} x_{jf}$. Note that, whereas $\mathcal{S}\mathcal{X}$ corresponds to a local information exchange between a given node and its direct neighbors, applying this linear operator recursively aggregates information from nodes in further hops along the graph. Namely, the {\it $k$-shifted signal} $\mathcal{S}^k\mathcal{X}$ aggregates information from nodes that are $k$-hops away. 
%

\textbf{Graph Convolutional Networks} (GCNs) stack multiple graph (convolutional) layers to extract high-level node representations from graph signals \citep{gama2019stability}. To formally describe the creation of convolutional features in a GCN with $L \in \mathbb{N}$ layers, let the output of layer $\ell-1$ (and thus the input to the $\ell^{\text{th}}$ layer) be $\mathcal{X}_{\ell-1}$, comprising of $F_{\ell-1}$ feature vectors. The input to layer $0$ is $\mathcal{X}_0=\mathcal{X}$, yielding $F_0 = F$, while the final output has $G := F_L$ features. Given a set of {\it filters} $A := \{\mathcal{A}_{\ell k} \in \mathbb{R}^{F_{\ell-1} \times F_{\ell}}| k \in [K] \cup\{0\}, \ell \in [L]\}$, the {\it graph (convolutional) filter} at layer $\ell$ processes the $F_{\ell-1}$ features of the input graph signal $\mathcal{X}_{\ell-1}$ in parallel by $F_\ell$ graph filters to output the following $F_\ell$ convolutional features:
\begin{equation}
    \label{eq:conv-features}
         \Psi^{\mathcal{A}_{\ell}}(\mathcal{X};\mathcal{S}) := \sum_{k=0}^K \mathcal{S}^k \mathcal{X}_{\ell-1} \mathcal{A}_{\ell k},
\end{equation}
where $\mathcal{A}_{\ell} := (\mathcal{A}_{\ell k})_{k=0}^K$ concatenates the filters at layer $\ell$. Note that all nodes share the same parameters to weigh equally the information from all $k$-hop away neighbors for any $k \in [K]$, which corresponds to the \textit{parameter sharing} of GCNs. A GCN is then defined as a nonlinear mapping between graph signals $\Phi: \mathbb{R}^{n \times F} \rightarrow \mathbb{R}^{n \times G}$, which is obtained by feeding the aggregated features through a pointwise nonlinear activation function $\sigma:\mathbb{R}\rightarrow\mathbb{R}$ (which, for brevity, denotes its entrywise application in \eqref{eq:graph-neural-network}), producing the feature vectors of the $\ell^{\text{th}}$ layer's output:
\begin{equation}
    \label{eq:graph-neural-network}
    \Phi(\mathcal{X}; \mathcal{S}, A) := \mathcal{X}_L; \quad \mathcal{X}_\ell := \sigma(\Psi^{\mathcal{A}_{\ell}}(\mathcal{X};\mathcal{S})).
\end{equation}
We denote the $f^{\text{th}}$ column of the filters $\mathcal{A}_{\ell k}$ as $\mathbf{a}^f_{\ell k} = [a_{\ell k}^{fg}]_{g \in [F_{\ell-1}]}$. Given some positive constants $\{R_{\ell}\}_{\ell \in [L]}$, we consider the following {\it non-convex} function class of (regularized) GCNs:
\begin{equation}
    \label{eq:gnn-model}
    \mathcal{F}_\text{gcn} := \{\Phi\text{ as in \eqref{eq:graph-neural-network}} : \max_{0 \leq k \leq K} \|\mathbf{a}^f_{\ell k}\|_2 \leq R_{\ell} \text{ } \forall f,\ell\}.
\end{equation}  
By setting $R_\ell=\infty$ for any $\ell \in [L]$, this function class reduces to the class of (unregularized) GCNs.



\subsection{Training GCNs}
\label{sec:erm}
We focus on the \textit{graph classification} task with $G$ classes. Here, we are given a training set $\mathcal{T}:=\{(\mathcal{G}^{(j)}, \mathcal{X}^{(j)}, y^{(j)})\}_{j=1}^{m}$ of $m$ samples, where the $j^{\text{th}}$ sample consists of a graph $\mathcal{G}^{(j)}$, a graph signal $\mathcal{X}^{(j)} \in \mathbb{R}^{n \times F}$ over $\mathcal{G}^{(j)}$ and the graph's true label $y^{(j)} \in [G]$. Without loss of generality, we assume that each graph $\mathcal{G}^{(j)}$ has $n$ nodes (otherwise, we can add isolated dummy nodes). 
We then train a GCN (as formulated in \eqref{eq:graph-neural-network}), while node features from the last layer are aggregated to obtain the graph-level representation via a fixed permutation invariant function, such as summation or a more sophisticated graph-level pooling function \citep{ying2018hierarchical,zhang2018end}. Training is done by solving the following optimization problem for some loss function $J: \mathbb{R}^{n\times G} \times [G] \rightarrow \mathbb{R}$:
\begin{equation}
    \label{eq:erm}
         \Phi^*_\text{gcn} \in \argmin_{\Phi \in \mathcal{F}_\text{gcn}} \frac{1}{m}\sum_{j=1}^m J(\Phi(\mathcal{X}^{(j)};\mathcal{S},A), y^{(j)}).
\end{equation}
We assume that $J$ is convex and $M$-Lipschitz in its first argument (predictions) for any fixed second argument (labels). The number of parameters to be trained are determined by the number of layers $L$, the number of filters $K$ and features $F_\ell$ in each layer $\ell$. By \eqref{eq:graph-neural-network}, each GCN $\Phi \in \mathcal{F}_\text{gcn}$ depends nonlinearly on the parameters, making the optimization problem in \eqref{eq:erm} non-convex. Next, we thus present a relaxation of the class $\mathcal{F}_\text{gcn}$, enabling a convex formulation of the associated problem in \eqref{eq:erm}.

\begin{remark}
    As our focus is on understanding generalization of convexified message-passing GNNs, we restricted evaluations to graph classification. This is exactly as in prior works that only consider a single task (see, e.g., \citep{lachi2025bridging,chen2023nagphormer,wu2022nodeformer}). However, \textbf{our convexification framework directly extends to node-level prediction}, by using the node features from the last layer instead of aggregating them to obtain the graph-level representation. While this results in a node-level instead of graph-level loss, our convexification procedure remains identical.
\end{remark}

\section{CONVEXIFICATION OF GCNS}
\label{sec:convex}
As $\mathcal{F}_\text{gcn}$ in \eqref{eq:gnn-model} is a non-convex set and the activation function is nonlinear, the ERM problem on GCNs is non-convex, raising the challenges mentioned in Section \ref{sec:intro}. To overcome them, we propose a procedure for making both the ERM's predictions domain and loss function \textit{convex}. Following \citet{zhang2017convexified}, we develop the class of {\bf \textit{Convexified GCNs} (CGCNs)}. In Section \ref{sec:linear}, we first depict our method for \textit{linear} activation functions. Given that the linear activation function would collapse all layers into one, this case is of no interest in practice. However, it supplies us with some insight into the general kernelization scheme of GCNs. In the nonlinear case (Section \ref{sec:nonlinear}), we reduce the problem to the linear case by proposing a relaxation to a suitably chosen {\it Reproducing Kernel Hilbert Space} (RKHS), enabling us to obtain a convex formulation of the associated ERM problem. We then present an algorithm for learning CGCNs (Section \ref{sec:The Algorithm for Learning CGCNs}), discuss its scalability and time complexity (Section \ref{sec:Scalable Learning of CGCNs and time complexity}), and establish its generalization guarantees (Section \ref{sec:Convergence and Generalization Error}).

\subsection{\textit{Linear} Activation Functions}
\label{sec:linear}
To provide the intuition for our method, we employ the nuclear norm in the simple case of the linear activation function $\sigma(\mathbf{w})=\mathbf{w}$, which provides us with a convex function class, unlike \eqref{eq:gnn-model}. In this case, the output of layer $\ell$ is simply $\mathcal{X}_{\ell} = \sum_{k=0}^K \mathcal{S}^k \mathcal{X}_{\ell-1} \mathcal{A}_{\ell k}$ due to \eqref{eq:conv-features}-\eqref{eq:graph-neural-network}. Hence, $\mathcal{X}_{\ell}$ linearly depends on the filters $\{\mathcal{A}_{\ell k}\}_{k=0}^K$ of layer $\ell$, i.e., the convolution function $\Psi^{\mathcal{A}_{\ell}}: \mathbb{R}^{n \times F_{\ell-1}} \rightarrow \mathbb{R}^{n \times F_{\ell}}$ in \eqref{eq:conv-features} linearly depends on the parameters.

\textbf{Low-Rank Constraints.} The filters matrix $\mathcal{A}_{\ell k} \in \mathbb{R}^{F_{\ell-1} \times F_{\ell}}$ has rank at most $F_{\ell-1}$. As $\mathcal{A}_\ell := (\mathcal{A}_{\ell k})_{k=0}^K$ concatenates the filters at layer $\ell$, then $\rk(\mathcal{A}_{\ell}) \leq F_{\ell-1}$. Hence, the class $\mathcal{F}_\text{gcn}$ of GCNs in \eqref{eq:gnn-model} corresponds to a collection of functions depending on norm and low-rank constraints imposed on the filters. We next define:
\begin{subequations}
    \label{eq:gnn-model alternate}
    \begin{align}
        \mathcal{F}_\text{linear-gcn} := \{\Psi^{\mathcal{A}_\ell} : & \max_{0 \leq k \leq K} \|\mathbf{a}^f_{\ell k}\|_2 \leq R_{\ell} \quad \text{ and } \label{eq:bound cons} \\
        & \rk(\mathcal{A}_{\ell}) = F_{\ell-1} \quad \forall  f,\ell\}, \label{eq:rank cons}
    \end{align}    
\end{subequations} 
which is the specific form of our original class of GCNs in \eqref{eq:gnn-model} for the linear case. While \eqref{eq:bound cons} holds even when the parameters are \textit{not} shared over all nodes, the low-rank constraint \eqref{eq:rank cons} will no longer be satisfied in this case. Namely, \eqref{eq:rank cons} is achieved via GCNs' parameter sharing: each node processes its neighborhood up to $K$ hops away using the same filters, and thus the number of independent rows of parameters remains at most $F_{\ell-1}$, regardless of the number of nodes in the graph.

\textbf{Nuclear Norm Relaxation.} However, filters that satisfy the constraints \eqref{eq:bound cons} and \eqref{eq:rank cons} form a \textit{non-convex} set. A standard convex relaxation of a rank constraint involves replacing it with a nuclear norm constraint. That is, controlling the sum of singular values instead of the number of non-zero ones, similarly to how the $\ell_1$-norm constraint can be viewed as a convex relaxation of the $\ell_0$-norm constraint.

\textbf{Upper Bounding the Nuclear-Norm.} In Appendix \ref{supp:nuclear norm const}, we prove that:
\begin{lemma}
    \label{lemma:nuclear norm const}
    Assume that the filters $\mathcal{A}_{\ell}$ at each layer $\ell$ satisfy \eqref{eq:bound cons} and \eqref{eq:rank cons}. Then, $\mathcal{A}_{\ell}$'s nuclear norm is upper bounded as $\|\mathcal{A}_{\ell}\|_* \leq \mathcal{B}_\ell$ for some $\mathcal{B}_\ell>0$ that is a function of $R_\ell , F_\ell, F_{\ell-1}, K$.
\end{lemma}

Using Lemma \ref{lemma:nuclear norm const}, we can replace \eqref{eq:bound cons} and \eqref{eq:rank cons} with a nuclear norm constraint, thus forming a convex function class. Namely, we define the class of \textbf{\textit{Convexified GCNs}} (CGCNs) as:
\begin{equation}
    \label{eq:cgcn}
    \mathcal{F}_\text{cgcn} := \{\{\Psi^{\mathcal{A}_\ell}\}_{\ell=1}^L : \|\mathcal{A}_{\ell}\|_* \leq \mathcal{B}_\ell \quad \forall\ell\},
\end{equation}
while guaranteeing that $\mathcal{F}_\text{gcn} \subseteq \mathcal{F}_\text{cgcn}$. Thus, solving the ERM in \eqref{eq:erm} over $\mathcal{F}_\text{cgcn}$ instead of $\mathcal{F}_\text{gcn}$ yields a convex optimization problem over a wealthier class of functions. Next, for the more general setting of \textit{nonlinear} activation functions, we devise an algorithm capable of solving this convex program. 

\subsection{\textit{Nonlinear} Activation Functions}
\label{sec:nonlinear}
For certain nonlinear activation functions $\sigma$ and suitably chosen kernel function $\kappa$ such as the Gaussian RBF kernel, we next prove how the class $\mathcal{F}_\text{gcn}$ of GCNs in \eqref{eq:gnn-model} can be relaxed to a {\it Reproducing Kernel Hilbert Space} (RKHS), reducing the problem to the linear activation case (Readers may refer to \citep{scholkopf2001generalized} for a brief overview on RKHS). Beforehand, we rephrase the filter at layer $\ell$ from \eqref{eq:conv-features}-\eqref{eq:graph-neural-network} as follows. Denoting the $f^{\text{th}}$ column of the filters $\mathcal{A}_{\ell k}$ as $\mathbf{a}^f_{\ell k}$, in Appendix \ref{supp:rephrase} we prove that:
\begin{lemma}
    \label{lemma:filter rephrased}
    A GCN induces, at each layer $\ell$, a nonlinear filter $\tau_\ell^{f}$ for the $f^{\text{th}}$ feature, which operates on the structured aggregation of neighborhood information and is given as (the activation $\sigma$ is applied entrywise): 
\begin{equation}
    \label{eq:filter map}
    \begin{array}{c}
         \tau_\ell^{f}: \mathcal{Z} \in \mathbb{R}^{(K+1) \times F_{\ell-1}} \mapsto \sigma(\sum_{k=0}^K \langle [\mathcal{Z}]_{k+1}, \mathbf{a}^f_{\ell k} \rangle).
    \end{array}    
\end{equation}
\end{lemma}

For any graph signal $\mathcal{S}^k\mathcal{X}$, let $\mathbf{z}_i^k(\mathcal{X}) := [\mathcal{S}^k\mathcal{X}]_i$ be the $i^{\text{th}}$ row of $\mathcal{S}^k\mathcal{X}$, which is the information aggregated at node $i$ from $k$-hops away. Defining the matrix $\mathcal{Z}_{i,\ell}(\mathcal{X}_{\ell-1}) \in \mathbb{R}^{(K+1) \times F_{\ell-1}}$ whose $(k+1)^{\text{th}}$ row is $[\mathcal{Z}_{i,\ell}(\mathcal{X}_{\ell-1})]_{k+1} = \mathbf{z}_i^k(\mathcal{X}_{\ell-1})$ for $k \in [K]\cup \{0\}$, the output at layer $\ell$ is thus given by $\tau_\ell^k(\mathcal{Z}_{i,\ell}(\mathcal{X}_{\ell-1}))$.

\textbf{Kernelization of GCNs.} In the sequel, we consider a training data $\mathcal{T} = \{(\mathcal{G}^{(j)}, \mathcal{X}^{(j)}, y^{(j)})\}_{j=1}^{m}$, some layer of index $\ell$ and a positive semi-definite kernel function $\kappa_{\ell}: \mathbb{R}^{F_{\ell-1}}\times\mathbb{R}^{F_{\ell-1}} \rightarrow \mathbb{R}$ for layer $\ell$. Thus, for brevity, we omit indices involving $\ell$ and $\mathcal{X}_{\ell-1}^{(j)}$ when no ambiguity arise. Namely, we denote, e.g., $\mathcal{Z}_{i,j} := \mathcal{Z}_{i,\ell}(\mathcal{X}_{\ell-1}^{(j)})$ and $\mathbf{z}_{i,j}^k := \mathbf{z}_i^k(\mathcal{X}_{\ell-1}^{(j)})$. 

First, we prove in Lemma \ref{lemma:ker} that the filter specified by \eqref{eq:filter map} resides in the RKHS induced by the kernel (See Appendix \ref{supp:kernelization} for a proof that adapts the one by \citet{zhang2017convexified} to GCNs). \citet{zhang2017convexified} prove that convolutional neural networks can be contained within the RKHS induced by properly chosen kernel and activation functions. For certain choices of kernels (e.g., the Gaussian RBF kernel) and a sufficiently smooth $\sigma$ (e.g., smoothed hinge loss), we can similarly prove that the filter $\tau_\ell^{f}$ in \eqref{eq:filter map} is contained within the RKHS induced by the kernel function $\kappa_\ell$. See Appendices \ref{sec:inverse}-\ref{sec:gauss} for possible choices of kernels and Appendix \ref{sec:valid act} for valid activation functions, as well as an elaboration on sufficient smoothness in Remark \ref{remark:sufficiently smooth}.

\begin{lemma}
    \label{lemma:ker}
    Mercer’s theorem \citep[Theorem 4.49]{steinwart2008support} implies that the filter $\tau^{f}_\ell$ at layer $\ell$ for the $f^{\text{th}}$ feature can be reparametrized as a linear combination of kernel products, reducing the problem to the linear activations case. Formally:
    \begin{equation}
        \label{eq:reparametrized}
        \begin{aligned}
            \tau^{f}_\ell(\mathcal{Z}_{i,j}) =
            \sum_{k=0}^K\sum_{(i',j') \in [n] \times [m]} \alpha_{k,(i',j')}^f \cdot \kappa_\ell (\mathbf{z}_{i,j}^k, \mathbf{z}_{i',j'}^{k}),
        \end{aligned}
    \end{equation}
    where $\{\alpha_{k,(i',j')}^f\}_{(i',j',k) \in [n] \times [m] \times ([K] \cup \{0\})}$ are some coefficients.
\end{lemma}


\begin{remark}
    Lemma \ref{lemma:ker} suggests that filters of the form \eqref{eq:reparametrized} are parameterized by a finite set of coefficients $\{\alpha_{k,(i',j')}^f\}$, allowing optimization to be carried over these coefficients instead of the original parameters $\mathcal{A}_\ell$ (See Section \ref{sec:The Algorithm for Learning CGCNs} for details on solving this optimization problem).
\end{remark}

To avoid rederiving all results in the kernelized setting, we reduce the problem to the linear case in Section \ref{sec:linear}. Thus, consider the (symmetric) kernel matrix $\mathcal{K}^k_\ell \in \mathbb{R}^{nm \times nm}$, whose entry at row $(i,j)$ and column $(i',j')$ equals to $\kappa_\ell (\mathbf{z}_{i,j}^k, \mathbf{z}_{i',j'}^{k})$. We factorize the kernel matrix as $\mathcal{K}^k_\ell = Q_\ell^k (Q_\ell^k)^\top$ with $Q_\ell^k \in \mathbb{R}^{nm \times P}$ (as in \citep{zhang2017convexified}, we can use, e.g., Cholesky factorization \citep{dereniowski2003cholesky} with $P = nm$). 
\begin{remark}
    \label{remark:feature}
    The $(i,j)$-th row of $Q_\ell^k$, i.e., $\mathbf{q}_{i,j}^k := [Q_\ell^k]_{(i,j)} \in \mathbb{R}^P$, can be interpreted as a feature vector instead of the original $\mathbf{z}_{i,j}^k \in \mathbb{R}^{F_{\ell-1}}$ (the information aggregated at node $i$ from $k$-hops away).
\end{remark}

Thus, using Remark \ref{remark:feature}, the reparameterization in \eqref{eq:reparametrized} of the filter at layer $\ell$ for the $f^{\text{th}}$ feature can be rephrased as:
\begin{equation}
    \label{eq:reparametrized rephrased}
    \begin{aligned}
        \tau^{f}_\ell(\mathcal{Z}_{i,j}) &= \sum_{k=0}^K \langle\mathbf{q}_{i,j}^k, \hat{\mathbf{a}}_{\ell k}^f \rangle ; \\
        \hat{\mathbf{a}}_{\ell k}^f :&= \sum_{(i',j',k) \in [n] \times [m] \times ([K] \cup \{0\})} \alpha_{k,(i',j')}^f \mathbf{q}_{i',j'}^k
    \end{aligned}
\end{equation}

Next, we outline the learning and testing phases of CGCNs (See Appendix \ref{supp:kernelization} for more details).

\begin{remark}
    \label{remark:convexified}
    Relaxing the filters to the induced RKHS reduces the problem to the linear case. The resulting relaxed filters satisfy a low-rank constraint and have a bounded nuclear norm. By Section \ref{sec:linear}, convexly relaxing the low-rank constraint via the nuclear norm yields a convex function class. 
\end{remark}

To summarize, we describe the forward pass operation of a single layer in Algorithm \ref{alg:forward pass}. It processes an input graph signal $\mathcal{X}$ and transforms it in two steps:

\textbf{Step 1: Arranging Input to the Layer.}  
Each layer $\ell$ takes as input the graph signal $\mathcal{X}_{\ell-1}$ from the previous layer and computes intermediate graph signals $Z_\ell^k(\mathcal{X}_{\ell-1})$ for different hop distances $k$ (line \ref{line:compute_Z}). For each node $i$, information is aggregated from its $k$-hop neighbors (line \ref{line:kernel_eval}), from which the $i^{\text{th}}$ row of $Z_\ell^k(\mathcal{X}_{\ell-1})$ is then obtained based on kernel products (lines \ref{line:kernel_eval}-\ref{line:pseudo_inverse}).

\textbf{Step 2: Computing the Layer Output.}  
Using the intermediate representations and the trainable filters $\hat{\mathcal{A}}_{\ell k}$, the layer's output signal $\mathcal{X}_\ell$ is calculated (line \ref{line:construct_output}). Specifically, each node’s output feature vector at layer $\ell$ is computed via a weighted sum of the transformed signals (line \ref{line:weighted_sum}).

\begin{algorithm}[t!]
    \caption{Forward Pass of the $\ell^\text{th}$ CGCN Layer}
    \textbf{Input:} Graph signal $\mathcal{X} \in \mathbb{R}^{n \times F}$, filters $\hat{\mathcal{A}}_{\ell} = (\hat{\mathcal{A}}_{\ell k})_{k=0}^K$
    \begin{algorithmic}[1]        
            \State {\textbf{\underline{Step 1: Arrange input to layer $\ell$}}}
            \State {Given an input graph signal $\mathcal{X}_{\ell-1} \in \mathbb{R}^{n \times F}$ from the previous layer, compute a new graph signal $Z_\ell^k(\mathcal{X}_{\ell-1}) \in \mathbb{R}^{n \times P}$ for any $\ell \in [L]$ and $k \in [K] \cup \{0\}$ as follows.} \label{line:compute_Z}
                    \State {Compute $\mathbf{z}_i^k(\mathcal{X}_{\ell-1}^{(j)}) := [\mathcal{S}^k \mathcal{X}_{\ell-1}^{(j)}]_i$, the signal aggregated at node $i$ from $k$-hops away with respect to the training graph signal $\mathcal{X}^{(j)}$.}
                    \State {As usual in kernel methods, we need to compute $\kappa_{\ell}(\mathbf{z}, \mathbf{z}_i^k(\mathcal{X}_{\ell-1}))$, where $\mathbf{z}\in \mathbb{R}^{F_{\ell-1}}$ is some input graph signal.} \label{line:kernel_eval}
                    \State {Form the vector $\mathbf{v}_{\ell}^k(\mathbf{z}_i^k(\mathcal{X}_{\ell-1}))$ whose elements are kernel products (See Appendix \ref{supp:kernelization}).}
                    \State {Compute the $i^{\text{th}}$ row of $Z_\ell^k(\mathcal{X}_{\ell-1})$ using $Z_{\ell}^k(\mathcal{X}_{\ell-1})_i = (Q^k)^\dag \mathbf{v}_{\ell}^k(\mathbf{z}_i^k(\mathcal{X}_{\ell-1}))$, where $(Q^k)^\dag$ is the pseudo-inverse of $Q^k$ (See Appendix \ref{supp:kernelization} for more details).} \label{line:pseudo_inverse}
            \State {\textbf{\underline{Step 2: Compute the layer's output}}}
            \State {Using the computed signals $Z_\ell^k(\mathcal{X}_{\ell-1})$, construct the output graph signal $\mathcal{X}_\ell$ for layer $\ell$.} \label{line:construct_output}
                    \State {Layer $\ell$ induces a \textit{linear} mapping between graph signals $\hat{\Psi}^{\hat{\mathcal{A}}_{\ell}}: \mathbb{R}^{n \times F_{\ell-1}} \rightarrow \mathbb{R}^{n \times F_\ell}$ s.t. the value of the $f^{\text{th}}$ feature of node $i$ at layer $\ell$ is $[\mathcal{X}_{\ell}]_{if} = [\hat{\Psi}^{\hat{\mathcal{A}}_{\ell}}(\mathcal{X}_{\ell-1})]_{if} = \sum_{k=0}^{K} \langle \hat{\mathbf{z}}_i^k(\mathcal{X}_{\ell-1}), \hat{\mathbf{a}}_{\ell k}^f \rangle$, where $\hat{\mathbf{a}}_{\ell k}^f$ are the trainable parameters (filters) of the CGCN layer from \eqref{eq:reparametrized rephrased}.} \label{line:weighted_sum}
    \end{algorithmic}
    \textbf{Output:} Output graph signal $\mathcal{X}_{\ell} \in \mathbb{R}^{n \times F_\ell}$
    \label{alg:forward pass}
\end{algorithm}

\subsection{The Algorithm for Learning Multi-Layer CGCNs}
\label{sec:The Algorithm for Learning CGCNs}

\begin{algorithm}[t!]
    \caption{Training of CGCNs}
    \label{alg:two-layer}
    \textbf{Input}: Training set $\mathcal{T} = \{(\mathcal{G}^{(j)}, \mathcal{X}^{(j)}, y^{(j)})\}_{j=1}^{m}$; Kernels $\{\kappa_\ell\}_{\ell=1}^L$; Regularization parameters $\{\mathcal{B}_\ell\}_{\ell=1}^L$; Number of graph filters $\{F_\ell\}_{\ell=1}^L$
    \begin{algorithmic}[1] 
        \For{\textbf{each} layer $2 \leq \ell \leq L$}
            \State{Taking $\{(\mathcal{G}^{(j)}, \mathcal{X}_{\ell-1}^{(j)}, {y}^{(j)})\}_{j=1}^{m}$ as training samples, construct a kernel matrix $\mathcal{K}_\ell^k \in \mathbb{R}^{nm \times nm}$ as in Section \ref{sec:nonlinear}.\label{state:kernel}} 
            \State{Compute a factorization or an approximation of $\mathcal{K}_\ell^k$ to obtain $Q_\ell^k \in \mathbb{R}^{nm \times P}$.\label{state:factorization}}
            \State{For any sample $\mathcal{X}_{\ell-1}^{(j)}$, form a graph signal $Z_\ell^k(\mathcal{X}_{\ell-1}^{(j)}) \in \mathbb{R}^{n \times P}$ via line \ref{line:pseudo_inverse} of Algorithm \ref{alg:forward pass}.}
            \State{Solve \eqref{eq:erm-convex} via \eqref{eq:projected gradient descent} to obtain a the filters $\hat{\mathcal{A}}_{\ell}$.}
            \State{Compute the output $\mathcal{X}_{\ell}^{(j)}$ of layer $\ell$ via Algorithm \ref{alg:forward pass}.}
        \EndFor
    \end{algorithmic}
    \textbf{Output}: Predictors $\{\hat{\Psi}^{\hat{\mathcal{A}}_{\ell}}(\cdot)\}_{\ell=1}^L$ and the output $\mathcal{X}_{L}^{(j)}$.
\end{algorithm}

Algorithm \ref{alg:two-layer} summarizes the training process of a multi-layer CGCN, estimating the parameters of each layer from bottom to top (i.e., in a layer-wise fashion). As such, to learn the nonlinear filter $\tau_\ell^f$ at layer $\ell$, we should learn the $P$-dimensional vector $\hat{\mathbf{a}}_{\ell k}^f$. For this sake, we define the graph signal $Z_j^k \in \mathbb{R}^{n \times P}$ for any $k \in [K] \cup \{0\}$ and $ j \in [m]$ whose $i^{\text{th}}$ row is $[Q^k]_{(i,j)}$. We can then apply the nuclear norm relaxation from Section \ref{sec:linear}. Solving the resulting optimization problem yields a matrix of filters $\hat{\mathcal{A}}_{\ell k} \in \mathbb{R}^{P \times F_{\ell}}$ for layer $\ell$ whose $f^{\text{th}}$ column is $\hat{\mathbf{a}}_{\ell k}^f$. 

Given regularization parameters $\{\mathcal{B}_\ell\}_{\ell=1}^L$ for the filters' nuclear norms, we thus aim to solve the following optimization problem at each layer $\ell$ due to line \ref{line:weighted_sum} of Algorithm \ref{alg:forward pass}:
\begin{equation}
    \label{eq:erm-convex}
    \begin{aligned}
        &\hat{\mathcal{A}}_{\ell} \in \argmin_{\|\hat{\mathcal{A}}_{\ell}\|_{*} \leq \mathcal{B}_\ell} \tilde{J}(\hat{\mathcal{A}}_{\ell}), \\ \text{ where }& \tilde{J}(\hat{\mathcal{A}}_{\ell}) := \frac{1}{m}\sum_{j=1}^m J(\hat{\Psi}^{\hat{\mathcal{A}}_{\ell}}(\mathcal{X}_{\ell-1}^{(j)}), y^{(j)}).
    \end{aligned}
\end{equation}
This can be easily solved via {\it projected gradient descent}: at round $t$, with step size $\eta^t > 0$, compute:
\begin{equation}
    \label{eq:projected gradient descent}
    \hat{\mathcal{A}}_{\ell}^{t+1} = \Pi_{\mathcal{B}_\ell}(\hat{\mathcal{A}}_{\ell}^{t} - \eta^t \nabla_{\hat{\mathcal{A}}_{\ell}} \tilde{J}(\hat{\mathcal{A}}_{\ell}^{t}))
\end{equation}
where $\nabla_{\hat{\mathcal{A}}_{\ell}} \tilde{J}$ is the gradient of $\tilde{J}$ from \eqref{eq:erm-convex} and $\Pi_{\mathcal{B}_\ell}$ denotes the Euclidean projection onto the nuclear norm ball $\{ \hat{\mathcal{A}}_{\ell} : \|\hat{\mathcal{A}}_{\ell}\|_{*} \leq \mathcal{B}_\ell \}$. This projection can be done by first computing the singular value decomposition of $\hat{\mathcal{A}}$ and then projecting the vector of singular values onto the $\ell_1$-ball via an efficient algorithm (see, e.g., \citep{duchi2008efficient,duchi2011adaptive,xiao2014proximal}). The problem in \eqref{eq:erm-convex} can also be solved using a projected version of any other optimizer (e.g., Adam \citep{kingma2014adam}).

\subsection{Scalability, Time Complexity and Challenges of Learning CGCNs}
\label{sec:Scalable Learning of CGCNs and time complexity}
Algorithm \ref{alg:two-layer}'s runtime highly depends on how the factorization matrix $Q_\ell^k$ is computed and its width $P$ (line \ref{state:factorization}). Naively picking $P=nm$ requires solving the exact kernelized problem, which can be costly. To gain scalability, we compute a \textit{Nystr\"{o}m approximation} \citep{williams2001using} in $\mathcal{O}(P^2 nm)$ time by randomly sampling training examples, computing their kernel matrix and representing each training example by its similarity to the sampled examples and kernel matrix. This avoids computing or storing the full kernel matrix, thus requiring only $\mathcal{O}(Pmn)$ space. While this approach is efficient \citep{williams2001using,drineas2005approximating}, it relies on the sampled subset adequately representing the entire data; poorly represented examples may degrade approximation quality.
See Appendix \ref{supp:nystrom} for more details and alternative methods.


\subsection{Convergence and Generalization Error}
\label{sec:Convergence and Generalization Error}
We now analyze Algorithm \ref{alg:two-layer}'s generalization error, focusing on two-layer models ($L=1$) in the binary classification case where the output dimension is $G=2$. Our main result can be easily extended to the multi-class case by a standard one-versus-all reduction to the binary case. In case of multiple hidden layers, our analysis then applies to each hidden layer separately. However, a global depth-wide bound is an important future direction.

Specifically, by an approach similar to \citet{zhang2017convexified}, we prove that the generalization error of Algorithm \ref{alg:two-layer} converges to the optimal generalization error of a GCN, as it is bounded above by this optimal error plus an additive term that decays polynomially to zero with the sample size. For any $\ell \in [L]$ and $k \in [K]$, we use the random kernel matrix $\mathcal{K}_{\ell}^k(\mathcal{X}) \in \mathbb{R}^{n \times n}$ induced by a graph signal $\mathcal{X} \in \mathbb{R}^{n \times F}$ drawn randomly from the population, i.e., the $(i,i')$-th entry of $\mathcal{K}_{\ell}^k(\mathcal{X})$ is $\kappa_\ell (\mathbf{z}_i^k(\mathcal{X}), \mathbf{z}_{i'}^{k}(\mathcal{X}))$, where $\mathbf{z}_i^k(\mathcal{X}_{\ell-1}) := [\mathcal{S}^k\mathcal{X}_{\ell-1}]_i$ aggregates information at node $i$ from $k$-hops away.
\begin{theorem}
    \label{thm:generalization error main}
    Consider two-layer models in the binary classification case (i.e., $L=1,G=2$). Let $J(\cdot,\mathbf{y})$ be $M$-Lipchitz continuous for any fixed $\mathbf{y} \in [G]^n$ and let $\kappa_L$ be either the inverse polynomial kernel or the Gaussian RBF kernel (Recall Appendices \ref{sec:inverse}-\ref{sec:gauss}). For any valid activation function $\sigma$ (Recall Appendix \ref{sec:valid act}), there is a constant radius $R_L>0$ such that the expected generalization error is at most:
    \begin{equation*}
        \label{eq:rademacher theory fin}
        \begin{aligned}
            \mathbb{E}_{\mathcal{X},\mathbf{y}} [J(\hat{\Psi}^{\hat{\mathcal{A}}_{L}}(\mathcal{X}), \mathbf{y})] &\leq \inf_{\Phi \in \mathcal{F}_\text{gcn}} \mathbb{E}_{\mathcal{X},\mathbf{y}} [J(\Phi(\mathcal{X}), \mathbf{y})] \\
            &\quad+ \frac{Const}{\sqrt{m}},
        \end{aligned}
    \end{equation*}
    where $Const$ is some constant that depends on $R_L$, but not on the sample size $m$.
\end{theorem}
\begin{proof}
    (\textit{Sketch})
    The full proof in Appendix \ref{supp:proof of thm 1} comprises of two steps: When $L=1,G=2$, we first consider a wealthier function class, denoted as $\mathcal{F}_\text{cgcn}$, which is proven to contain the class of GCNs $\mathcal{F}_\text{gcn}$ from \eqref{eq:graph-neural-network}. The richer class relaxes the class of GCNs in two aspects: (1) The graph filters are relaxed to belong to the RKHS induced by the kernel, and (2) the $\ell_2$-norm constraints in \eqref{eq:graph-neural-network} are substituted by a single constraint on the Hilbert norm induced by the kernel. Though Algorithm \ref{alg:two-layer} never explicitly optimizes over the function class $\mathcal{F}_\text{cgcn}$, we prove that the predictor $\hat{\Psi}^{\hat{\mathcal{A}}_{L}}$ generated at the last layer is an empirical risk minimizer in $\mathcal{F}_\text{cgcn}$. Then, we prove that $\mathcal{F}_\text{cgcn}$ is not "too big" by upper bounding its \textit{Rademacher complexity}, a key concept in empirical process theory that can be used to bound the generalization error in our ERM problem (see, e.g., \citep{bartlett2002rademacher} for a brief overview on Rademacher complexity). Combining this upper bound with standard Rademacher complexity theory \citep{bartlett2002rademacher}, we infer that the generalization error of $\hat{\Psi}^{\hat{\mathcal{A}}_{L}}$ converges to the best generalization error of $\mathcal{F}_\text{cgcn}$. Since $\mathcal{F}_\text{gcn} \subset \mathcal{F}_\text{cgcn}$, the latter error is bounded by that of GCNs, completing the proof.
\end{proof}

\section{EXPERIMENTAL EVALUATION}
\label{sec:Experimental evaluation}
\textbf{Datasets.} We evaluate our framework through extensive experiments on several benchmark graph classification datasets \citep{morris2020tudataset}, including the four bioinformatics datasets MUTAG, PROTEINS, PTC-MR, NCI1 and the chemical compounds dataset Mutagenicity. We also evaluated our models on the \verb|ogbg-molhiv| molecular property prediction dataset from the Open Graph Benchmark \citep{hu2020open}, originally sourced from \textsc{MoleculeNet} \citep{wu2018moleculenet} and among its largest benchmarks. Following many prior works \citep{ma2019graph,ying2018graph}, we randomly split each dataset into three parts: 80\% as training set, 10\% as validation set and the remaining 10\% as test set. The statistics of the above datasets appear in Appendix \ref{supp:Dataset Descriptions}.

\textbf{Baselines.} We apply our convexification procedure to the message-passing mechanisms of popular GNNs and hybrid graph transformers, creating convex counterparts. Our framework is integrated into a wide range of leading models, including GCN \citep{defferrard2016convolutional,kipf2017semisupervised}, GAT \citep{velivckovic2018graph}, GATv2 \citep{brody2022how}, GIN \citep{xu2018how}, GraphSAGE \citep{hamilton2017inductive}, ResGatedGCN \citep{bresson2017residual}, as well as the hybrid transformers GraphGPS \citep{rampavsek2022recipe} and GraphTrans \citep{shi2021masked,wu2021representing}. In Appendix \ref{supp:Experimental Results for DirGNNs}, we include additional experiments for DirGNN \citep{rossi2024edge}. All models share the same architecture design for fair comparison; see Appendix \ref{supp:Implementation Details} for more details. \textbf{Our code is available at \citep{code_implementation}.}

We would like to clarify that many recent graph learning papers in influential venues continue using the same baselines or older ones (see, e.g., \citep{armgaan2025gnnxemplar,yehudai2025depthwidth,hordan2025spectral,ma2023graph,chen2023nagphormer}). Further, our baselines are still listed in leaderboards of modern datasets like Open Graph Benchmark (OGB) \citep{hu2020open}. Thus, our baselines match what current peer-reviewed GNN work uses for molecular datasets and graph classification, especially for message-passing comparisons. We also clarify that our goal is parallel: to compare convexified GNNs to their own non-convex counterparts, \textit{not} to propose a new SOTA model.

\begin{table*}[t!]
\centering
\caption{Accuracy ($\%\pm\sigma$) of $2$-layer convex versions of base models against their $2$- and $6$-layer non-convex counterparts on various graph classification datasets over $4$ runs, each with different seeds. Datasets are ordered from left to right in ascending order of size. Convex variants start with 'C'. A $2$-layer model is marked by (2L), while a $6$-layer one by (6L). Results of convex models surpassing their non-convex variants by a large gap ($\sim$10--40\%) are in \textbf{bold} and modest improvements are \underline{underlined}.}
\begin{tabular}{lccccc}
\toprule
\textbf{Model} & \textbf{MUTAG} & \textbf{PTC-MR} & \textbf{PROTEINS} & \textbf{NCI1} & \textbf{Mutagenicity} \\
\midrule
\rowcolor{Gainsboro!50} CGCN (2L)           & \textbf{$\mathbf{85.0\pm3.5}$} & \underline{${67.6\pm5.8}$} & {$\mathbf{75.0\pm3.2}$} & {$\mathbf{77.1\pm2.8}$} & \underline{${80.6\pm3.3}$} \\
GCN (2L)            & $50.0\pm20.3$ & $48.5\pm10.4$ & $59.5\pm4.1$ & $61.0\pm4.0$ & $69.7\pm2.4$ \\
GCN (6L)            & $65.0\pm19.6$ & $63.0\pm11.8$ & $58.7\pm6.1$ & $65.0\pm5.9$ & $72.7\pm1.6$ \\
\midrule
\rowcolor{Gainsboro!50} CGIN (2L)           & \textbf{$\mathbf{82.5\pm5.5}$} &{$\mathbf{66.4\pm3.4}$} & {$\mathbf{77.7\pm1.8}$} & \underline{${74.8\pm9.4}$} & {$\mathbf{77.9\pm4.5}$} \\
GIN (2L)            & $50.0\pm11.7$ & $49.3\pm10.0$ & $50.4\pm12.0$ & $61.0\pm1.9$ & $67.8\pm4.6$ \\
GIN (6L)            & $61.2\pm14.3$ & $51.9\pm14.5$ & $61.6\pm1.2$ & $66.3\pm2.4$ & $69.1\pm3.2$ \\
\midrule
\rowcolor{Gainsboro!50} CGAT (2L)           & \textbf{$\mathbf{83.7\pm5.4}$} & {$\mathbf{70.1\pm9.9}$} & {$\mathbf{72.1\pm4.7}$} & \underline{${68.5\pm4.7}$} & \underline{${75.3\pm2.6}$} \\
GAT (2L)            & $62.0\pm22.7$ & $58.7\pm18.3$ & $60.0\pm3.3$ & $67.1\pm3.9$ & $69.3\pm3.0$ \\
GAT (6L)            & $75.0\pm16.2$ & $61.0\pm17.6$ & $63.1\pm1.8$ & $67.3\pm2.6$ & $72.8\pm2.0$ \\
\midrule
\rowcolor{Gainsboro!50} CGATv2 (2L)         & \textbf{$\mathbf{87.5\pm2.5}$} & {$\mathbf{69.4\pm7.5}$} & {$\mathbf{73.4\pm1.7}$} & \underline{${66.2\pm1.5}$} & \underline{${80.6\pm1.0}$} \\
GATv2 (2L)          & $68.7\pm4.1$ & $56.3\pm13.6$ & $61.6\pm1.6$ & $62.2\pm6.2$ & $70.1\pm2.7$ \\
GATv2 (6L)          & $70.0\pm7.9$ & $59.2\pm17.4$ & $63.3\pm2.1$ & $63.5\pm4.0$ & $74.9\pm1.9$ \\
\midrule
\rowcolor{Gainsboro!50} CResGatedGCN (2L)   & \textbf{$\mathbf{83.7\pm5.4}$} & {$\mathbf{73.0\pm5.0}$} & {$\mathbf{78.5\pm2.2}$} & \underline{${69.8\pm3.2}$} &  \underline{${79.1\pm3.5}$} \\
ResGatedGCN (2L)    & $43.7\pm17.4$ & $50.7\pm7.1$ & $47.9\pm9.9$ & $67.2\pm2.7$ & $70.6\pm0.9$ \\
ResGatedGCN (6L)    & $55.0\pm12.2$ & $55.4\pm20.5$ & $55.8\pm8.7$ & $67.3\pm2.5$ & $71.5\pm1.6$ \\
\midrule
\rowcolor{Gainsboro!50} CGraphSAGE (2L)     & \textbf{$\mathbf{82.5\pm2.4}$} & {$\mathbf{72.3\pm 6.1}$} & {$\mathbf{73.6\pm3.1}$} & \underline{${69.3\pm2.8}$} & {$\mathbf{80.1\pm3.6}$} \\
GraphSAGE (2L)      & $65.0\pm16.5$ & $56.2\pm17.0$ & $62.5\pm0.8$ & $61.4\pm4.1$ & $70.6\pm1.8$ \\
GraphSAGE (6L)      & $55.0\pm12.7$ & $61.5\pm8.9$ & $63.1\pm1.8$ & $64.7\pm2.0$ & $70.8\pm2.3$ \\
\midrule
\rowcolor{Gainsboro!50} CGraphTrans (2L)    & \textbf{$\mathbf{81.2\pm4.1}$} & {$\mathbf{68.1\pm8.3}$} & {$\mathbf{75.4\pm3.2}$} & \underline{${68.9\pm2.5}$} & \underline{${73.8\pm3.4}$} \\
GraphTrans (2L)     & $62.5\pm16.7$ & $55.3\pm13.8$ & $60.7\pm2.6$ & $68.9\pm0.8$ & $68.0\pm4.1$ \\
GraphTrans (6L)     & $57.5\pm19.5$ & $61.4\pm16.9$ & $61.6\pm1.2$ & $70.0\pm4.3$ & $72.1\pm3.8$ \\
\midrule
\rowcolor{Gainsboro!50} CGraphGPS (2L)      & \textbf{$\mathbf{83.7\pm2.1}$} & {$\mathbf{65.5\pm9.3}$} & {$\mathbf{81.6\pm7.2}$} & {$\mathbf{60.1\pm6.1}$} & \underline{${59.4\pm4.1}$} \\
GraphGPS (2L)       & $50.9\pm19.6$ & $49.2\pm20.1$ & $42.0\pm1.9$ & $50.8\pm2.6$ & $41.4\pm8.7$ \\
GraphGPS (6L)       & $52.5\pm2.1$ & $46.4\pm2.3$ & $48.8\pm1.5$ & $49.4\pm1.9$ & $55.1\pm5.5$ \\
\bottomrule
\end{tabular}
\label{tab:convex-vs-non-convex}
\end{table*}

\textbf{Results and Analysis.} Table \ref{tab:convex-vs-non-convex} presents our main empirical results for the graph classification benchmarks, comparing two-layer convex models against two-layer and six-layer non-convex models. On the \verb|ogbg-molhiv| dataset, our convex models consistently obtained around 98\% accuracy, surpassing their non-convex counterparts, which reached about 96\%. Impressively, in most cases, the two-layer convex variants substantially surpass their non-convex counterparts by 10–40\% accuracy, showcasing the practical strength of our framework alongside its theoretical benefits. Even when improvements are modest, our convex models match or slightly surpass their non-convex versions, underscoring the robustness and broad applicability of our approach. Notable examples are GATv2 and GraphGPS, whose convex versions significantly outperform their non-convex ones in 4 out of 5 datasets, with slight improvements on Mutagenicity. Our convex models' superiority over the hybrid transformer GraphGPS, which combines message-passing with self-attention, highlights the power of convexified GNNs over non-convex ones, even when enhanced by self-attention. In Appendix \ref{supp:Additional Experiments}, we visualize these trends more holistically.

Over-parameterization is known to aid gradient-based methods to reach global minima that generalize well in non-convex models \citep{allen2019convergence,du2018gradient,zou2020gradient}, making it a common practice for obtaining strong performance. Our results challenge this approach, showing that our shallow convex models often surpass their more over-parameterized non-convex counterparts and obtain better results with fewer layers and parameters. Namely, we surpass them in both performance and model compactness. These gains are most notable in small-data regimes, which are prone to overfitting, where we suspect that local optima obtained by non-convex models may hinder generalization. However, convex models avoid such pitfalls, offering more stable and reliable learning. Our results show that rich representations can be attained without over-parameterization by exploiting convexity in shallow models. 

In Appendix \ref{supp:Experimental Results for Deeper Convex Models}, we give more results for \textit{\textbf{deeper}} convex models, depicting how accuracy can be improved as depth increases. This directly addresses how our convexified models handle over-smoothing: our results in \ref{supp:Experimental Results for Deeper Convex Models} show that the accuracy of convex models can further improve with additional layers, which is consistent with our theoretical guarantees in Theorem \ref{thm:generalization error main}. Hence, this illustrates that errors remain controlled as depth increases, offering strong empirical support for the practical effectiveness of multi-layer CGNNs, which is crucial because state-of-the-art GNN applications often rely on deeper architectures. Intuitively, convexified models handle over-smoothing well thanks to the nuclear norm relaxation used by our model. Specifically, the nuclear norm projection reduces the effective capacity of filters and, in practice, suppresses small singular-value directions that often correspond to high-frequency/noisy modes over the graph. This prevents representation collapse in shallow models and empirically reduces over-smoothing in deeper convex models. See Appendix \ref{supp:excel} for more details.

\section{CONCLUSION}
We presented CGNNs, which transform the training of message-passing GNNs to a convex optimization problem. They ensure global optimality and computational efficiency, while having provable generalization and statistical guarantees. Empirically, we showed that our CGNNs can effectively capture complex patterns without the need for over-parameterization or extensive data. This insight offers a more efficient and data-scarce alternative to traditional deep learning approaches, with substantial improvements across various benchmark datasets. As noted in Section \ref{sec:Scalable Learning of CGCNs and time complexity}, CGNNs' training is heavily affected by the factorization of the kernel matrix in terms of time and space complexity. Addressing these challenges is thus a key direction for future work.

\bibliography{sample}

\section*{Checklist}

\begin{enumerate}

  \item For all models and algorithms presented, check if you include:
  \begin{enumerate}
    \item A clear description of the mathematical setting, assumptions, algorithm, and/or model. [Yes. Our mathematical setting and model are described in Section \ref{sec:gnn}, while in Sections \ref{sec:nonlinear}--\ref{sec:Convergence and Generalization Error} we clarify that certain theoretical claims, such as inclusion in the RKHS and generalization guarantees in Theorem \ref{thm:generalization error main}, hold only for certain choices of kernel functions and sufficiently smooth activation functions (see Appendices \ref{sec:inverse}-\ref{sec:valid act}). Our algorithm for learning multi-layer CGCNs is fully described in Section \ref{sec:The Algorithm for Learning CGCNs}.]
    \item An analysis of the properties and complexity (time, space, sample size) of any algorithm. [Yes. Section \ref{sec:Scalable Learning of CGCNs and time complexity} and Appendix \ref{supp:nystrom} discuss the computational efficiency of the proposed algorithms and how they scale with dataset size.]
    \item (Optional) Anonymized source code, with specification of all dependencies, including external libraries. [No. The paper fully discloses all the information needed to reproduce the main experimental results. Further, all code is publicly available at \citep{code_implementation} for reproducibility as well as further inspection.]
  \end{enumerate}

  \item For any theoretical claim, check if you include:
  \begin{enumerate}
    \item Statements of the full set of assumptions of all theoretical results. [Yes. In Sections \ref{sec:nonlinear}--\ref{sec:Convergence and Generalization Error}, we clarify that certain theoretical claims, such as inclusion in the RKHS and generalization guarantees in Theorem \ref{thm:generalization error main}, hold only for certain choices of kernel functions (e.g., Gaussian RBF) and sufficiently smooth activation functions.]
    \item Complete proofs of all theoretical results. [Yes. Lemma \ref{lemma:nuclear norm const} is proven in Appendix \ref{supp:nuclear norm const}, in Appendix \ref{supp:rephrase} we prove Lemma \ref{lemma:filter rephrased}, Lemma \ref{lemma:ker} in proven in Appendix \ref{supp:kernelization}, and Theorem \ref{thm:generalization error main} is discussed in Section \ref{sec:Convergence and Generalization Error} along with a proof sketch, while a complete proof appears in Appendix \ref{supp:proof of thm 1}.]
    \item Clear explanations of any assumptions. [Yes. See Appendices \ref{sec:inverse}-\ref{sec:gauss} for possible choices of kernels and Appendix \ref{sec:valid act} for valid activation functions, as well as an elaboration on sufficient smoothness in Remark \ref{remark:sufficiently smooth}.]     
  \end{enumerate}

  \item For all figures and tables that present empirical results, check if you include:
  \begin{enumerate}
    \item The code, data, and instructions needed to reproduce the main experimental results (either in the supplemental material or as a URL). [No. The paper fully discloses all the information needed to reproduce the main experimental results. All code is publicly available at \citep{code_implementation} for reproducibility as well as further inspection.]
    \item All the training details (e.g., data splits, hyperparameters, how they were chosen). [Yes. In Section \ref{sec:Experimental evaluation} and Appendix~\ref{supp:Implementation Details}, we specify all the training and test details (e.g., data splits, hyperparameters, how they were chosen, type of optimizer, etc.). The statistics of the considered datasets appear in Appendix \ref{supp:Dataset Descriptions}, while Appendix~\ref{supp:Implementation Details} includes additional implementation details.]
    \item A clear definition of the specific measure or statistics and error bars (e.g., with respect to the random seed after running experiments multiple times). [Yes. In Table \ref{tab:convex-vs-non-convex} we explicitly report standard deviations. Similarly, the bar plots in Figure \ref{fig:accuracy bar plots} report mean accuracies over four runs along with error bars that denote the standard deviation.]
    \item A description of the computing infrastructure used. (e.g., type of GPUs, internal cluster, or cloud provider). [Yes. The paper provides sufficient information on the computing resources needed to reproduce the experiments in Appendix \ref{supp:Implementation Details}.]
  \end{enumerate}

  \item If you are using existing assets (e.g., code, data, models) or curating/releasing new assets, check if you include:
  \begin{enumerate}
    \item Citations of the creator If your work uses existing assets. [Yes. The creators and original owners of assets used in the paper are properly credited.]
    \item The license information of the assets, if applicable. [Yes. The license and terms of use are explicitly mentioned and properly respected. We also state which version of the asset is used and, where possible, include a URL.]
    \item New assets either in the supplemental material or as a URL, if applicable. [Not Applicable. The paper does not release new assets.]
    \item Information about consent from data providers/curators. [Yes. As mentioned earlier, the license and terms of use are explicitly mentioned and properly respected.]
    \item Discussion of sensible content if applicable, e.g., personally identifiable information or offensive content. [Not Applicable.]
  \end{enumerate}

  \item If you used crowdsourcing or conducted research with human subjects, check if you include:
  \begin{enumerate}
    \item The full text of instructions given to participants and screenshots. [Not Applicable. The paper does not involve crowdsourcing nor research with human subjects.]
    \item Descriptions of potential participant risks, with links to Institutional Review Board (IRB) approvals if applicable. [Not Applicable. The paper does not involve crowdsourcing nor research with human subjects.]
    \item The estimated hourly wage paid to participants and the total amount spent on participant compensation. [Not Applicable. The paper does not involve crowdsourcing nor research with human subjects.]
  \end{enumerate}

\end{enumerate}

\clearpage
\appendix
\thispagestyle{empty}

\onecolumn
\aistatstitle{Convexified Message-Passing Graph Neural Networks: \\
Supplementary Materials}

\section{ADDITIONAL RELATED WORK}
\label{supp:ADDITIONAL RELATED WORK}
\subsection{Graph Transformers} 
Instead of aggregating local neighborhoods, Graph Transformers capture interaction information between any pair of nodes via a single self-attention layer. Some works design specific attention mechanisms or positional encodings \citep{hussain2022global,kreuzer2021rethinking,ma2023graph,ying2021transformers}, while others combine message-passing GNNs to build hybrid architectures \citep{chen2022structure,rampavsek2022recipe,shi2021masked,wu2021representing}. Those methods enable nodes to interact with any other node, facilitating direct modeling of long-range relations. Our framework broadly applies to such hybrid methods, as we exhibit on the leading models GraphGPS \citep{rampavsek2022recipe} and GraphTrans \citep{shi2021masked,wu2021representing}. Empirically, applying our framework consistently improves performance in graph classification tasks, with significant gains on GraphGPS, showcasing the versatility and effectiveness of our approach.

\subsection{Comparison to Convexified CNNs \citep{zhang2017convexified}}
\label{supp:ccnns}
There is a fundamental mathematical gap between CNNs \citep{zhang2017convexified} and message-passing GNNs studied in our work. The transition from grids to graphs is not merely an application of existing theory; it requires solving a distinct set of theoretical problems related, which \cite{zhang2017convexified} do not address. Our work thus differs from that by \cite{zhang2017convexified} in several fundamental dimensions, such as:
\begin{enumerate}
    \item \textbf{Irregular Domains:} \cite{zhang2017convexified} rely on the fixed, grid-like structure of images where "neighbors" are deterministic; the input has a constant dimension and regular structure. They cannot account for a domain where the "receptive field" (neighborhood) size varies from node to node as in GNNs. Further, \cite{zhang2017convexified} cannot handle information exchange between nodes. In contrast, our work focuses on irregular domains consisting of graphs with varying numbers of nodes and edges, which cannot be derived from the fixed-grid analysis of Zhang et al. This lack of stationarity breaks many simplifications used in the CNN setting.

    \item \textbf{Differences in Convexification Procedures:} \cite{zhang2017convexified} use RKHS constructions tied to convolutional filters and patch extraction on grids. The kernel acts on image patches and the analysis exploits fixed patch geometry. Conversely, our work requires a different RKHS construction designed to explicitly model the recursive aggregation of information across multiple hops.

    \item \textbf{Nuclear norm relaxation:} Unlike \cite{zhang2017convexified}, our nuclear norm relaxation is applied to matrix representations of graph filters that concatenate hop-wise filters across features. The low-rank structure thus has node-hop-feature semantics (not just patches in a grid).
\end{enumerate}

Finally, we show how the convexification can be applied to a wide family of message-passing GNNs and hybrid transformer models, which cannot be treated by \cite{zhang2017convexified}.

\subsection{Fundamental Distinctions from \cite{cohen2021convex}}
\label{supp:compare to myself}

While our work and that by \cite{cohen2021convex} use convexification in graph contexts, the architectural foundations, scope and mathematical formulations are fundamentally different.
\begin{itemize}
    \item \textbf{Modeling and Scope:} \cite{cohen2021convex} proposed convex versions of specific GNN models termed as aggregation GNN, which apply only to a signal with temporal structure; particularly, their formulation is restricted to a \textbf{fixed} communication topology. Their work is also narrowly tailored and focused on a specific application in swarm robotics, where convexity arises from modeling assumptions tailored to domain adaptation. In contrast, our CGNN framework establishes a general and theoretically principled convexification of diverse message-passing GNNs by transforming the standard architecture itself, independent of application-specific constraints. In other words, we resolve key mathematical limitations of their work:
    \begin{enumerate}
        \item \cite{cohen2021convex}: Their formulation is mathematically restricted to a \textbf{fixed communication topology}. Consequently, \textbf{their method is undefined if the number of nodes changes between training and inference}. It cannot be applied to standard graph classification benchmarks.
        \item Our CGNN Framework: In contrast, our framework is explicitly designed for \textbf{multiple, variable-sized graphs}. By relaxing the message-passing filters directly (rather than polynomial coefficients of a fixed operator), our method supports training on \textbf{multiple, heterogeneous graphs}. As detailed in Appendix \ref{supp:Experimental Details}, our implementation handles datasets with varying numbers of nodes and edges, \textbf{a capability that is mathematically impossible under the formulation of \cite{cohen2021convex}}.
    \end{enumerate}

    \item \textbf{Modularity and Generalization:} Unlike \cite{cohen2021convex}, whose formulation does not preserve standard GNN modularity and is limited to a \textbf{single fixed graph and a specific control setting}, our framework retains the usual GNN architectural structure and supports \textbf{multiple} graphs during training. \textbf{Our convexification also extends to, e.g., node-level tasks} by applying the same procedure with a node-level loss on the final-layer features; no further reformulation is required. In contrast, their model cannot be directly adapted to generic node or graph prediction without substantial redesign. Hence, CGNNs provide a scalable, domain-agnostic convexification framework with a fundamentally different mathematical basis. In summary, the key modularity differences are:
    \begin{enumerate}
        \item \cite{cohen2021convex}: Narrowly tailored to a specific control setting (swarm robotics) and requires substantial redesign to adapt to other tasks. 
        \item Our Work: Our framework preserves the modularity of standard GNNs; extending our convexification to node-level tasks (as opposed to graph-level) requires only a change in the final loss function, not a reformulation of the architecture.
    \end{enumerate}

    \item \textbf{Tighter Theoretical Foundations and Different Models:} Our nuclear norm relaxation step and derived upper bound on relaxed filters’ nuclear norm are fundamentally different, much tighter and more general than those presented by \cite{cohen2021convex}.
    \begin{enumerate}
        \item The analysis by \cite{cohen2021convex} is limited to \textbf{polynomial} filters. After the nuclear norm relaxation of \textbf{polynomial} filters, \cite{cohen2021convex} prove that the nuclear norm of the filters at layer $\ell$ is at most $R_{\ell} F_{\ell} F_{\ell-1} \sqrt{\sum_i (\sum_k \lambda_{i}^{k})^2}$. Here, $\lambda_1, \dots, \lambda_n$ are the distinct eigenvalues of the shift operator $\mathcal{S}$, $R_{\ell},F_{\ell},K$ are as in Section \ref{sec:gnn}.

        \item Our analysis is fundamentally different: it is \textbf{not restricted to polynomial filters}, but applies to \textbf{general} filters. After the nuclear norm relaxation, Lemma \ref{lemma:nuclear norm const} gives an upper bound on the filters’ nuclear norm, whose derivation is simpler and yields a much tighter upper bound for general filters. Namely, we prove that the nuclear norm of \textbf{\textit{any}} possible filter at layer $\ell$ is at most $R_{\ell} \sqrt{F_{\ell} F_{\ell-1} (K+1)}$. Eventually, this results in a smaller generalization gap bound, as observed in Theorem \ref{thm:generalization error main}.
    \end{enumerate}

    \item \textbf{Algorithmic Differences:} Hence, our relaxation to a reproducing kernel Hilbert space (RKHS) and its analysis also substantially differ from \cite{cohen2021convex}. 
    \begin{enumerate}
        \item The methods of \cite{cohen2021convex} are limited to learning a relaxation of the frequency representation of \textbf{polynomial} filters. Further, their algorithm is restricted to using classic projected gradient descent instead of more modern optimizers. Also, the implementation of their framework is confined to graphs with the \textbf{same} number of nodes.
        \item As before, our analysis \textbf{is not restricted to polynomial filters}. Rather, it applies to \textbf{general non-convex} filters, by directly learning the relaxed filters themselves. Our algorithm can also employ a projected version of any modern optimizer (e.g., Adam). Additionally, as detailed in Appendix \ref{supp:Experimental Details}, our implementation utilizes several components that allow use to train our convex models on datasets with graphs containing \textbf{varying numbers of nodes and edges}.
    \end{enumerate}
    In short, the algorithmic differences are:
    \begin{enumerate}
        \item \cite{cohen2021convex}: \textbf{Restricted to learning a relaxation of the frequency representation of polynomial filters} using classic projected gradient descent. 
        \item Our Work: Our relaxation to a Reproducing Kernel Hilbert Space (RKHS) allows us to \textbf{directly learn the relaxed filters using any modern optimizer} (e.g., Adam), ensuring better scalability and convergence in practice.
    \end{enumerate}
\end{itemize}

\subsection{GNNs as Unified Optimization Problems}
\label{supp:GNNs as Unified Optimization Problems}
Our convexification framework relates to and fundamentally differs from prior work on GNNs as unified optimization problems. \cite{chen2023bridging} first classify spatial GNNs based on how they aggregate neighbor information, while discussing how graph convolution can be depicted as an optimization problem. They also discuss how filters of spectral GNNs can be approximated via either linear, polynomial or rational approximation. The works by \cite{he2021bernnet}, \cite{wang2022powerful} and \cite{guo2023graph} focus on polynomial approximation, treating spectral filters as parameterized by a polynomial basis. Specifically: 
\begin{itemize}
    \item BernNet \citep{he2021bernnet} parameterizes an arbitrary spectral filter in a Bernstein polynomial basis. BernNet is analyzed from the perspective of graph optimization, showing that any polynomial filter that attempts to approximate a valid filter has to take the form of BernNet.
    \item \cite{wang2022powerful} notice that spectral GNNs with different polynomial bases of the spectral filters have the same expressive power but different empirical performance. Thus, for **linear** GNNs, they analyze the effect of polynomial basis from an optimization perspective. Specifically, they show that a set of orthonormal bases whose weight function corresponds to the graph signal density can enable linear GNNs to achieve the highest convergence rate.
    \item \cite{guo2023graph} continue that work, by first proposing for learning an orthonormal polynomial basis. Then, they tackle the definition of optimal basis from an optimization perspective in [3], proposing a model that computes the optimal basis for a given graph structure and graph signal.

\end{itemize}

However, in all those works training remains \textit{\textbf{non-convex}}. Our contribution is thus orthogonal and complementary: rather than deriving architectures, we \textbf{\textit{convexify the training problem}} itself for general, non-convex message-passing GNNs (not only spectral ones). Training thus becomes a convex optimization problem, ensuring global optimality as well as an accurate and rigorous analysis of statistical properties. Further, our framework applies directly to general message-passing GNNs and hybrid transformer architectures like GraphGPS and GraphTrans, whereas The works by \cite{he2021bernnet}, \cite{wang2022powerful} and \cite{guo2023graph} only cover spectral GNNs. In fact, the optimization-oriented analysis of \cite{wang2022powerful} is restricted to \textit{linear} GNNs, while we focus on convexifying \textbf{\textit{non-convex}} GNNs. We will add a dedicated discussion explicitly contrasting these settings.

\subsection{Classic Graph Kernels}
\label{supp:graph kernels}
\begin{enumerate}
    \item \textbf{The relation to Weisfeiler-Lehman:} The most direct link lies in the Weisfeiler-Lehman (WL) graph kernel \citep{shervashidze2011weisfeiler}. It is well-established that standard message-passing GNNs are functionally equivalent to the 1-WL isomorphism test in terms of expressive power \citep{xu2018how}. Classical WL kernels operate by explicitly mapping graphs to high-dimensional sparse feature vectors containing counts of color-refined subtrees, which are then typically optimized via Support Vector Machines (SVMs).
    \begin{itemize}
        \item \textbf{The WL kernel approach:} The kernel method separates feature extraction (WL refinement) from learning (SVM). The "learning" is convex (quadratic programming), but the features are fixed and discrete.
        \item \textbf{Our Approach:} Our CGNN framework internalizes this process. By mapping the nonlinear message-passing filters themselves into an RKHS, we do not merely extract fixed features; we learn the propagation mechanism itself. Crucially, we achieve this while retaining the hallmark benefit of the SVM era: \textbf{\textit{convexity}}. CGNNs thus guarantee global optima, eliminating the local minima issues plaguing standard non-convex GNN training.
    \end{itemize}

    \item \textbf{Implicit vs. Explicit Infinite-Dimensional Spaces:} Classical graph kernels rely on the kernel trick to implicitly operate in high-dimensional spaces without computing coordinates. Conversely, widely-used GNNs (e.g., GCN, GIN) operate in explicit, low-dimensional Euclidean spaces. Convexified GNNs occupy a middle ground that leverages the best of both worlds. We utilize the RKHS formalism to first parameterize graph filters in a countable-dimensional feature space. By the representer theorem, we give a reduction of the original ERM problem to a \textit{\textbf{finite}}-dimensional one. Then, we solve the learning problem via projected gradient descent. This mirrors the Infinite-Width GNN (or Neural Tangent Kernel) regime, but with a tractable, finite-data implementation. While Graph Neural Tangent Kernels (GNTK) \citep{du2019graph,krishnagopal2023graph} are often static and used as fixed priors, CGNNs remain adaptable, learning the representation within the convex constraint set.

    \item \textbf{Dataset Implications and Generalization:} The connection is empirically relevant given the standard evaluation benchmarks. Datasets such as TUDataset (e.g., MUTAG, PROTEINS, NCI1) were originally curated to benchmark graph kernels. For years, kernel methods (like the WL-subtree kernel and Deep Graph Kernels) outperformed early GNNs on these tasks due to the small sample sizes, where over-parameterized GNNs were prone to overfitting. By convexifying GNNs, our method introduces rigorous regularization inherent to the RKHS norm. This explains the strong performance of CGNNs on these “kernel-dominated" datasets: we provide the adaptivity of a neural network but with the complexity control and generalization guarantees typically associated with SVM-based kernel methods. Thus, CGNNs can be viewed as the natural neural evolution of the WL-kernel: preserving the global optimization guarantees while enabling end-to-end learning of continuous message-passing dynamics.

\end{enumerate}

\newpage
\section{UPPER BOUNDING THE NUCLEAR NORM}
\label{supp:nuclear norm const}

\begin{customlemma}{\ref{lemma:nuclear norm const}}
    For any layer $\ell$, $\|\mathcal{A}_{\ell}\|_* \leq \mathcal{B}_\ell$ for some constant $\mathcal{B}_\ell>0$ dependent on $R_\ell , F_\ell, F_{\ell-1}, K$.
\end{customlemma}
\begin{proof}
    Note that the nuclear norm of any matrix $B$ can be bounded as $\|B\|_* \leq \|B\|_F \sqrt{\rk(B)}$, where $\|B\|_F$ is the Frobenius norm which is identical to the standard Euclidean norm of the vectorized version of the matrix. Hence, the nuclear norm of the filters $\mathcal{A}_{\ell}$ satisfying \eqref{eq:bound cons} and \eqref{eq:rank cons} is upper bounded by 
    \[
    \|\mathcal{A}_{\ell}\|_*^2 \leq \|\mathcal{A}_{\ell }\|_F^2 F_{\ell-1} = {\sum_{k=0}^K \sum_{f \in [F_\ell]} \sum_{g \in [F_{\ell-1}]} |a_{\ell k}^{fg}|^2} \cdot {F_{\ell-1}} \leq R_\ell^2 {F_\ell F_{\ell-1} (K+1)}.
    \]
    That is, $\|\mathcal{A}_{\ell}\|_* \leq \mathcal{B}_\ell$ for $\mathcal{B}_\ell := R_\ell \sqrt{F_\ell F_{\ell-1} (K+1)}$.
\end{proof}

\section{CHOICE OF KERNELS AND VALID ACTIVATIONS}
\label{sec:choice-of-kernels}
In this section, we provide properties regarding two types of kernels: the \textit{inverse polynomial kernel} (Section \ref{sec:inverse}) and the \textit{Gaussian RBF kernel} (Section \ref{sec:gauss}). We show that the reproducing kernel Hilbert space (RKHS) corresponding to these kernels contain the nonlinear filter at each layer $\ell$ for the $f^{\text{th}}$ feature of node $i$ for certain choices of an activation function $\sigma$. Finally, we summarize the possible choices for an activation function $\sigma$ that result from our discussion (Section \ref{sec:valid act}).

\subsection{Inverse Polynomial Kernel}
\label{sec:inverse}
For $F$ features, the \textit{inverse polynomial kernel} $\kappa^{IP}:\mathbb{R}^{F}\times\mathbb{R}^{F} \rightarrow \mathbb{R}$ is given as follows:
\begin{equation}
    \label{eq:inverse-poly}
    \kappa^{IP}(\mathbf{z},\mathbf{z}') := \frac{1}{2 - \langle \mathbf{z},\mathbf{z}'\rangle}; \quad \|\mathbf{z}\|_2 \leq 1,\|\mathbf{z}'\|_2 \leq 1
\end{equation}
\citet{zhang2017convexified} prove that $\kappa^{IP}$ indeed constitutes a legal kernel function. We include the proof to make the paper self-contained: they provide a feature mapping $\varphi:\mathbb{R}^{F} \rightarrow \ell^2(\mathbb{N})$ for which $\kappa^{IP}(\mathbf{z},\mathbf{z}')=\langle \varphi(\mathbf{z}),\varphi(\mathbf{z}')\rangle$. The $(k_1,\dots,k_j)$-th coordinate of $\varphi(\mathbf{z})$, where $j \in \mathbb{N}$ and $k_1,\dots,k_j \in [F_\ell]$, is given by $2^{-\frac{j+1}{2}}x_{k_1}\dots x_{k_j}$. Thus:
\begin{equation}
    \label{eq:def-varphi}    
        \langle \varphi(\mathbf{z}),\varphi(\mathbf{z}')\rangle = \sum_{j=0}^\infty 2^{-(j+1)} \sum_{(k_1,\dots,k_j) \in [F_\ell]^j} z_{k_1}\dots z_{k_j} z'_{k_1}\dots z'_{k_j}
\end{equation}
Since $\|\mathbf{z}\|_2 \leq 1$ and $\|\mathbf{z}'\|_2\leq 1$, the series above is absolutely convergent. Combined with the fact that $|\langle \mathbf{z},\mathbf{z}'\rangle| \leq 1$, this yields a simplified form of \eqref{eq:def-varphi}:
\begin{equation}
    \label{eq:simlified-varphi}
        \langle \varphi(\mathbf{z}),\varphi(\mathbf{z}')\rangle = \sum_{j=0}^\infty 2^{-(j+1)} (\langle \mathbf{z},\mathbf{z}'\rangle)^j = \frac{1}{2 - \langle \mathbf{z},\mathbf{z}'\rangle} = \kappa^{IP}(\mathbf{z},\mathbf{z}') 
\end{equation}
as desired. Combining \eqref{eq:filter map} with the proof of Lemma 1 in Appendix A.1 of Zhang et al. \citep{zhang2016l1}, the RKHS associated with $\kappa^{IP}$ is comprised of the class of nonlinear graph filters in \eqref{eq:filter map}. It can be formulated as follows:
\begin{corollary}
\label{corollary:inverse-contains}
    Assume that $\sigma$ has a polynomial expansion $\sigma(t) = \sum_{j=0}^\infty a_j t^j$. Let $C_\sigma(t):=\sqrt{\sum_{j=0}^\infty 2^{j+1}a_j^2 t^{2j}}$. If $C_\sigma (\|\mathbf{a}^f_{\ell k}\|_2) < \infty$, then the RKHS induced by $\kappa^{IP}$ contains the function $\tau_\ell^{f}: \mathcal{Z} \in \mathbb{R}^{(K+1) \times F_{\ell-1}} \mapsto \sigma(\sum_{k=0}^K \langle [\mathcal{Z}]_{k+1}, \mathbf{a}^f_{\ell k} \rangle)$ with Hilbert norm $\|\tau_\ell^{f}\|_{\mathcal{A}} \leq  C_\sigma(\sum_{k=0}^K \|\mathbf{a}^f_{\ell k}\|_2)$.
\end{corollary}
Consequently, Zhang et al. \citep{zhang2016l1} state that upper bounding $C_\sigma(t)$ for a particular activation function $\sigma$ is sufficient. To enforce $C_\sigma(t) < \infty$, the scalars $\{a_j\}_{j=0}^\infty$ must rapidly converge to zero, implying that $\sigma$ must be sufficiently smooth. For polynomial functions of degree $d$, the definition of $C_\sigma$ yields $C_\sigma(t)=\mathcal{O}(t^d)$. For the {\it sinusoid} activation $\sigma(t):=sin(t)$, we have:
\begin{equation}
    \label{eq:sinusoid}
    C_\sigma(t) = \sqrt{\sum_{j=0}^\infty \frac{2^{2j+2}}{((2j+1)!)^2}(t^2)^{2j+1}} \leq 2 e^{t^2}
\end{equation}
For the {\it erf} function $\sigma_{erf}(t):= \frac{2}{\sqrt{\pi}}\int_0^t e^{z^2}dz$ and the {\it smoothed hinge loss} function $\sigma_{sh}(t):= \int_{-\infty}^t \frac{1}{2}(\sigma_{erf}(t) + 1)dz$, Zhang et al. \citep{zhang2016l1} proved that $C_\sigma(t)=\mathcal{O}(e^{ct^2})$ for universal numerical constant $c > 0$.
\begin{remark}
\label{remark:sufficiently smooth}
Generally, a function $f$ is \textit{\textbf{sufficiently smooth}} if it is differentiable {\em sufficient} number of times. For instance, if we were to consider $\sigma(t):=\sin(t)$, its taylor expansion is provided by $\sin(t) = \sum_{j = 0}^\infty \frac{(-1)^j}{(2j+1)!}t^{2j+1}$. Denoting $b_j := \frac{(-1)^j}{(2j+1)!}t^{2j+1}$, the ratio test yields $\lim_{j \rightarrow \infty}|\frac{b_{j+1}}{b_{j}}| = 0$ for all $t \in I\!R$. Hence, the radius of convergence of the expansion is the set of all real numbers.

Regarding Corollary \ref{corollary:inverse-contains}, as well as Corollary \ref{corollary:rbf} which follows in the next subsection, for enforcing $C_\sigma(t) < \infty$, the coefficients $\{a_j\}_{j=0}^\infty$ in $\sigma$'s polynomial expansion must quickly converge to zero as illustrated above, thus requiring $\sigma$ to be sufficiently smooth. Such activations include {\bf polynomial} functions, the {\bf sinusoid}, the {\bf erf} function and the {\bf smoothed hinge loss}.
\end{remark}

\subsection{Gaussian RBF Kernel}
\label{sec:gauss}
For $F$ features, the \textit{Gaussian RBF kernel} $\kappa^{RBF}:\mathbb{R}^{F}\times\mathbb{R}^{F} \rightarrow \mathbb{R}$ is given as follows:
\begin{equation}
    \label{eq:rbf}
    \kappa^{RBF}(\mathbf{z},\mathbf{z}') := e^{-\gamma \|\mathbf{z}-\mathbf{z}'\|_2^2}; \quad \|\mathbf{z}\|_2 = \|\mathbf{z}'\|_2 = 1
\end{equation}
Combining \eqref{eq:filter map} with the proof of Lemma 2 in Appendix A.2 of Zhang et al. \citep{zhang2016l1}, the RKHS associated with $\kappa^{RBF}$ is comprised of the class of nonlinear input features. It can be formulated as follows:
\begin{corollary}
\label{corollary:rbf}
    Assume that $\sigma$ has a polynomial expansion $\sigma(t) = \sum_{j=0}^\infty a_j t^j$. Let $C_\sigma(t):=\sqrt{\sum_{j=0}^\infty \frac{j!e^{2\gamma}}{(2\gamma)^j} a_j^2 t^{2j}}$. If $C_\sigma (\|\mathbf{a}^f_{\ell k}\|_2) < \infty$, then the RKHS induced by $\kappa^{RBF}$ contains the function $\tau_\ell^{f}: \mathcal{Z} \in \mathbb{R}^{(K+1) \times F_{\ell-1}} \mapsto \sigma(\sum_{k=0}^K \langle [\mathcal{Z}]_{k+1}, \mathbf{a}^f_{\ell k} \rangle)$ with Hilbert norm $\|\tau_\ell^f\|_{\mathcal{A}} \leq  C_\sigma(\sum_{k=0}^K \|\mathbf{a}^f_{\ell k}\|_2)$.
\end{corollary}

Comparing Corollary \ref{corollary:inverse-contains} and Corollary \ref{corollary:rbf}, $\kappa^{RBF}$ requires a stronger condition on the smoothness of the activation function. For polynomial functions of degree $d$, $C_\sigma(t)=\mathcal{O}(t^d)$ remains intact. However, for the {\it sinusoid} activation $\sigma(t):=\sin(t)$, we have:
\begin{equation}
    \label{eq:sinusoid-rbf}
    C_\sigma(t) = \sqrt{e^{2\gamma}\sum_{j=0}^\infty \frac{1}{(2j+1)!}(\frac{t^2}{2\gamma})^{2j+1}} \leq e^{\frac{t^2}{4\gamma}+\gamma}
\end{equation}
In contrast, for both the {\it erf} function $\sigma_{erf}(t):= \frac{2}{\sqrt{\pi}}\int_0^t e^{z^2}dz$ and the {\it smoothed hinge loss} function $\sigma_{sh}(t):= \int_{-\infty}^t \frac{1}{2}(\sigma_{erf}(t) + 1)dz$, $C_\sigma(t)$ is \textbf{infinite}. Hence, \textbf{the RKHS induced by $\kappa^{RBF}$ does not contain input features activated by those two functions}.

\subsection{Valid Activation Functions}
\label{sec:valid act}
Now, we summarize the possible valid choices of an activation function $\sigma$:
\begin{enumerate}
    \item Arbitrary polynomial functions.

    \item The sinusoid activation function $\sigma(t):=\sin(t)$.

    \item The {\it erf} function $\sigma_{erf}(t):= \frac{2}{\sqrt{\pi}}\int_0^t e^{\tau^2} \mathrm{d}\tau$, which approximates the sigmoid function.

    \item The {\it smoothed hinge loss} function $\sigma_{sh}(t):= \int_{-\infty}^t \frac{1}{2}(\sigma_{erf}(\tau) + 1) \mathrm{d}\tau$, which approximates the ReLU function.
\end{enumerate}

Recall that the {\it erf} and {\it smoothed hinge loss} activation functions are unsuited for the RBF kernel.

\section{REPHRASING THE NONLINEAR FILTER AT EACH LAYER}
\label{supp:rephrase}

\begin{proof}
    For any graph signal $\mathcal{S}^k\mathcal{X}$, let $\mathbf{z}_i^k(\mathcal{X}) := [\mathcal{S}^k\mathcal{X}]_i$ be the $i^{\text{th}}$ row of $\mathcal{S}^k\mathcal{X}$, which is the information aggregated at node $i$ from $k$-hops away. We also denote the $f^{\text{th}}$ column of the filters $\mathcal{A}_{\ell k}$ as $\mathbf{a}^f_{\ell k} = [a_{\ell k}^{fg}]_{g \in [F_{\ell-1}]}$. By \eqref{eq:conv-features}, the value of the $f^{\text{th}}$ feature of node $i$ at layer $\ell$ can be thus written as:
    \begin{equation}
        \label{eq:f th feature}
        \begin{aligned}
            [\Psi^{\mathcal{A}_{\ell}}(\mathcal{X};\mathcal{S})]_{if} = \sum_{k=0}^K \sum_{g=1}^{F_{\ell-1}} [\mathbf{z}_i^k(\mathcal{X}_{\ell-1})]_{g} [\mathbf{a}^f_{\ell k}]_{g} = \sum_{k=0}^K \langle \mathbf{z}_i^k(\mathcal{X}_{\ell-1}), \mathbf{a}^f_{\ell k} \rangle.
        \end{aligned}    
    \end{equation}
\end{proof}

\section{KERNELIZATION FOR NONLINEAR ACTIVATIONS}
\label{supp:kernelization}
In this section, we elaborate on the kernelization for the case of nonlinear activation functions, briefly discussed in Section \ref{sec:nonlinear}. Recall that $\sigma$ denotes, for brevity, its entrywise application in \eqref{eq:graph-neural-network}. Further, the of the $f^{\text{th}}$ feature of node $i$ at layer $\ell$'s output is given by $\sigma(\sum_{k=0}^K \langle \mathbf{z}_i^k(\mathcal{X}_{\ell-1}), \mathbf{a}^f_{\ell k} \rangle)$ by \eqref{eq:f th feature}. According to Appendix \ref{sec:choice-of-kernels}, consider a sufficiently smooth valid activation function $\sigma$, and a proper choice of kernel $\kappa_{\ell}: \mathbb{R}^{F_{\ell-1}} \times \mathbb{R}^{F_{\ell-1}} \rightarrow \mathbb{R}$. Without loss of generality, we assume that, each input input graph signal $\mathcal{X}^{(j)}$, the feature vector $[\mathcal{X}^{(j)}]_i = \mathbf{x}_i^{(j)}$ of each node $i$ resides in the unit $\ell_2$-ball (which can be obtain via normalization). Combined with $\kappa_\ell$'s being a continuous kernel with $\kappa_\ell(\mathbf{z}, \mathbf{z}) \leq 1$, Mercer’s theorem \citep[Theorem 4.49]{steinwart2008support} implies that there exists a feature mapping $\varphi_\ell : \mathbb{R}^{F_{\ell-1}} \rightarrow \ell^2(\mathbb{N})$ for which $\kappa_{\ell}(\mathbf{z},\mathbf{z}')=\langle \varphi_\ell(\mathbf{z}),\varphi_\ell(\mathbf{z}')\rangle$, which yields the following for each feature $f$ of any node $i$ at layer $\ell$:
\begin{equation}
    \label{eq:feature map}
    \sigma\Bigg(\sum_{k=0}^K \langle \mathbf{z}_i^k(\mathcal{X}_{\ell-1}), \mathbf{a}^f_{\ell k} \rangle\Bigg) = \sum_{k=0}^K \langle \varphi_\ell(\mathbf{z}_i^k(\mathcal{X}_{\ell-1})),\varphi_\ell(\mathbf{a}^f_{\ell k}) \rangle =: \sum_{k=0}^K \langle \varphi_\ell(\mathbf{z}_i^k(\mathcal{X}_{\ell-1})), \bar{\mathbf{a}}^f_{\ell k} \rangle
\end{equation}
where $\bar{\mathbf{a}}^f_{\ell k} \in \ell^2(\mathbb{N})$ is a countable-dimensional vector, due to the fact that $\varphi$ itself is a countable sequence of functions. By Corollary \ref{corollary:inverse-contains} and Corollary \ref{corollary:rbf}, we have that $\|\bar{\mathbf{a}}^f_{\ell k}\|_2 \leq C_{\sigma, \ell} (\|{\mathbf{a}}^f_{\ell k}\|_2)$, provided a monotonically increasing $C_\sigma$ which depends on the kernel $\kappa_\ell$. Accordingly, $\varphi_\ell(\mathbf{z})$ may be utilized as the vectorized representation of each output features $\mathbf{z}_i^k(\mathcal{X}_{\ell-1})$, where $\bar{\mathbf{a}}^f_{\ell k}$ constitutes the linear graph filters, thus reducing the problem to training a GCN with the identity activation function. 

Each graph filter is now parametrized by the \textit{countable}-dimensional vector $\bar{\mathbf{a}}^f_{\ell k}$. Subsequently, we supply a reduction of the original ERM problem to a \textit{finite}-dimensional one. Output on the training data $\mathcal{T}:=\{(\mathcal{G}^{(j)}, \mathcal{X}^{(j)}, y^{(j)})\}_{j=1}^{m}$ should be solely considered when solving the ERM problem, i.e., $\sum_{k=0}^K \langle \varphi_\ell(\mathbf{z}_i^k(\mathcal{X}_{\ell-1})), \bar{\mathbf{a}}^f_{\ell k} \rangle$ for each feature $f$ of any node $i$ at layer $\ell$. Let $\mathcal{P}_{\ell-1}^k$ be the orthogonal projector onto the linear subspace spanned by $\Gamma_{\ell-1}^k := \{\mathbf{z}_i^k(\mathcal{X}^{(j)}_{\ell-1}): i \in [n], j \in [m], k \in [K] \cup \{0\}\}$. Therefore, for any $(i,j) \in [n] \times [m]$:
\begin{equation}
    \label{eq:ortho-projector}
    \sum_{k=0}^K \langle \varphi_\ell(\mathbf{z}_i^k(\mathcal{X}_{\ell-1}^{(j)})), \bar{\mathbf{a}}^f_{\ell k} \rangle = \sum_{k=0}^K \langle \mathcal{P}_{\ell-1}^k \varphi_\ell(\mathbf{z}_i^k(\mathcal{X}_{\ell-1}^{(j)})), \bar{\mathbf{a}}^f_{\ell k} \rangle = \sum_{k=0}^K \langle \varphi_\ell(\mathbf{z}_i^k(\mathcal{X}_{\ell-1}^{(j)})), \mathcal{P}_{\ell-1}^k \bar{\mathbf{a}}^f_{\ell k} \rangle
\end{equation}
where the last step stems from the fact that the orthogonal projector $\mathcal{P}_{\ell-1}^k$ is self-adjoint. Without loss of generality, for the sake of solving the ERM, we assume that $\bar{\mathbf{a}}^f_{\ell k}$ belongs to $\Gamma_{\ell-1}^k$. Similarly to \eqref{eq:reparametrized}, we can thus reparameterize it by:
\begin{equation}
    \label{eq:w-bar-reparameterize}
    \bar{\mathbf{a}}^f_{\ell k} = \sum_{(i,j) \in [n] \times [m]} \beta_{\ell k,(i,j)}^f \varphi_\ell(\mathbf{z}_i^k(\mathcal{X}_{\ell-1}^{(j)}))
\end{equation}
where we denote $\mathbf{\beta}_{\ell k}^f \in \mathbb{R}^{nm}$ as a vector of the coefficients, whose $(i,j)$-th coordinate is $\beta_{\ell k,(i,j)}^f$. To estimate $\bar{\mathbf{a}}^f_{\ell k}$, it thus suffices to estimate the vector $\mathbf{\beta}_{\ell k}^f$. By definition, the relation $(\mathbf{\beta}_{\ell k}^f)^T \mathcal{K}_{\ell}^k \mathbf{\beta}_{\ell k}^f = \|\bar{\mathbf{a}}^f_{\ell k}\|_2^2$ holds, where $\mathcal{K}_{\ell}^k \in \mathbb{R}^{nm \times nm}$ is the symmetric kernel matrix from Section \ref{sec:nonlinear}, whose rows and columns are indexed by some triple $(i,j) \in [n] \times [m]$. The entry at row $(i,j)$ and column $(i',j')$ equals to $\kappa_\ell (\mathbf{z}_i^k(\mathcal{X}^{(j)}_{\ell-1}), \mathbf{z}_{i'}^{k}(\mathcal{X}^{(j')}_{\ell-1}))$. A suitable factorization of $\mathcal{K}_{\ell}^k$ such that $\mathcal{K}_{\ell}^k = Q_\ell^k (Q_\ell^k)^\top$ for $Q_\ell^k \in \mathbb{R}^{nm \times P}$ then yields the following norm constraint:
\begin{equation}
    \label{eq:norm-constraint}
    \begin{aligned}
        \|(Q_\ell^k)^\top \mathbf{\beta}_{\ell k}^f\|_2 = \sqrt{(\mathbf{\beta}_{\ell k}^f)^T \mathcal{K}_{\ell}^k \mathbf{\beta}_{\ell k}^f} = \|\bar{\mathbf{a}}^f_{\ell k}\|_2^2 \leq C_{\sigma,\ell} (\|{\mathbf{a}}^f_{\ell k}\|_2) \leq  C_{\sigma,\ell} (R_\ell)
    \end{aligned}   
\end{equation}
where the last inequality is due to $\|{\mathbf{a}}^f_{\ell k}\|_2 \leq R_\ell$ due to \eqref{eq:gnn-model}.

Let $\mathbf{v}_{\ell}^k(\mathbf{z}) \in \mathbb{R}^{nm}$ be a vector whose $(i,j)$-th coordinate equals to $\kappa_{\ell}(\mathbf{z},\mathbf{z}_i^k(\mathcal{X}_{\ell-1}^{(j)}))$. By \eqref{eq:feature map} and \eqref{eq:w-bar-reparameterize}, we can write:
\begin{equation}
    \label{eq:rkhs-sigma-alternate}
    \begin{aligned}
        \sigma\Bigg(\sum_{k=0}^K \langle \mathbf{z}_i^k(\mathcal{X}_{\ell-1}^{(j)}), \mathbf{a}^f_{\ell k} \rangle\Bigg) \equiv \sum_{k=0}^K \langle \varphi_\ell(\mathbf{z}_i^k(\mathcal{X}_{\ell-1}^{(j)})), \bar{\mathbf{a}}^f_{\ell k} \rangle \equiv \sum_{k=0}^K \langle \mathbf{v}_{\ell}^k(\mathbf{z}_i^k(\mathcal{X}_{\ell-1}^{(j)})), \mathbf{\beta}_{\ell k}^f \rangle \\
        \equiv \sum_{k=0}^K \langle \mathbf{\beta}_{\ell k}^f, \mathbf{v}_{\ell}^k(\mathbf{z}_i^k(\mathcal{X}_{\ell-1}^{(j)})) \rangle
    \end{aligned}
\end{equation}
For each $\mathbf{z}_i^k(\mathcal{X}_{\ell-1}^{(j)})$, the vector $\mathbf{v}_{\ell}^k(\mathbf{z}_i^k(\mathcal{X}_{\ell-1}^{(j)}))$ resides in the column space of the symmetric kernel matrix $\mathcal{K}_{\ell}^k$. Thus, given that $(Q_\ell^k)^\dag$ denotes the pseudo-inverse of $Q_\ell^k$, we infer the following for any $(i,j) \in [n] \times [m]$:
\begin{equation}
    \label{eq:pseudo-inverse}
    \begin{aligned}
        \sum_{k=0}^K \langle \mathbf{\beta}_{\ell k}^f, \mathbf{v}_{\ell}^k(\mathbf{z}_i^k(\mathcal{X}_{\ell-1}^{(j)})) \rangle = \sum_{k=0}^K  (\mathbf{\beta}_{\ell k}^f)^\top Q_\ell^k (Q_\ell^k)^\dag \mathbf{v}_{\ell}^k(\mathbf{z}_i^k(\mathcal{X}_{\ell-1}^{(j)})) \\
        = \sum_{k=0}^K \Big\langle ((Q_\ell^k)^\top)^\dag (Q_\ell^k)^\top \mathbf{\beta}_{\ell k}^f, \mathbf{v}_{\ell}^k(\mathbf{z}_i^k(\mathcal{X}_{\ell-1}^{(j)})) \Big\rangle
    \end{aligned}
\end{equation}
Accordingly, replacing the vector $\mathbf{\beta}_{\ell k}^f$ on the right-hand side of \eqref{eq:rkhs-sigma-alternate} by $((Q_\ell^k)^\top)^\dag (Q_\ell^k)^\top \mathbf{\beta}_{\ell k}^f$ won't affect the empirical risk. For any matrix $\mathcal{Z} \in \mathbb{R}^{(K+1) \times F_{\ell-1}}$ whose $(k+1)^{\text{th}}$ row is given by $[\mathcal{Z}]_{k+1} := \mathbf{z}_k$ for $k \in [K] \cup \{0\}$, we thereby propose the following parameterization function: 
\begin{equation}
    \label{eq:parameterization}
    \tau_\ell^f(\mathcal{Z}) := \sum_{k=0}^K \Big\langle ((Q_\ell^k)^\top)^\dag (Q_\ell^k)^\top \mathbf{\beta}_{\ell k}^f, \mathbf{v}_{\ell}^k(\mathbf{z}_k) \Big\rangle = \sum_{k=0}^K \langle (Q_\ell^k)^\dag \mathbf{v}_{\ell}^k(\mathbf{z}_k), (Q_\ell^k)^\top \mathbf{\beta}_{\ell k}^f \rangle
\end{equation}
Consequently, the value of the $f^{\text{th}}$ feature of node $i$ at layer $\ell$ is given by $\tau_\ell^f(\mathcal{Z}_{i,\ell}(\mathcal{X}_{\ell-1}))$, where $\mathcal{Z}_{i,\ell}(\mathcal{X}_{\ell-1}) \in \mathbb{R}^{(K+1) \times F_{\ell-1}}$ is the matrix whose $(k+1)^{\text{th}}$ row is $[\mathcal{Z}_{i,\ell}(\mathcal{X}_{\ell-1})]_k = \mathbf{z}_i^k(\mathcal{X}_{\ell-1})$ for $k \in [K]\cup \{0\}$. That is, $(Q_\ell^k)^\top \mathbf{\beta}_{\ell k}^f$ now act as the filters under the reparameterization in \eqref{eq:parameterization} which satisfy the following constraints:

\textbf{Norm Constraint.} $\|(Q_\ell^k)^\top \mathbf{\beta}_{\ell k}^f\|_2 \leq  C_{\sigma,\ell} (R_\ell)$ due to \eqref{eq:norm-constraint}.

\textbf{Rank Constraint.} Its rank is at most $F_{\ell-1}$.

Therefore, similarly to the linear case in Section \ref{sec:linear}, those constraints can be relaxed to the nuclear norm constraint:
\begin{equation}
    \label{eq:nonlinear nuclear}
    \|(Q_\ell^k)^\top \mathbf{\beta}_{\ell k}^f\|_* \leq C_{\sigma,\ell} (R_\ell) \sqrt{F_\ell F_{\ell-1} (K+1)}
\end{equation}

Comparing the constraint \eqref{eq:cgcn} from the linear activation function case with \eqref{eq:nonlinear nuclear}, we notice that the sole distinction is that the term $R_\ell$ is replaced by $C_{\sigma,\ell} (R_\ell)$, which is required due to our use of the kernel trick for reducing the case of nonlinear activation functions to the linear setting.

\section{MORE DETAILS ON MATRIX APPROXIMATION AND FACTORIZATION METHODS}
\label{supp:nystrom}
As mentioned in Section \ref{sec:Scalable Learning of CGCNs and time complexity}, the time complexity of Algorithm \ref{alg:two-layer} largely depends on the width $P$ of each factorization matrix $Q_\ell^k$ and the method used to obtain it (line \ref{state:factorization}). A naive choice of $P=nm$ corresponds to solving the full kernelized problem, which can be computationally expensive. For instance, using Cholesky factorization with $P =nm$ \citep{dereniowski2003cholesky} incurs a complexity of $\mathcal{O}(m^3 n^3)$ and space complexity of $\mathcal{O}(m^2 n^2)$.

However, to gain scalability, we use \textit{Nystr\"{o}m approximation} \citep{williams2001using}: an approximation $\mathcal{K}_{\ell}^k \approx Q_\ell^k (Q_\ell^k)^\top$ for $Q_\ell^k \in \mathbb{R}^{nm \times P}$ is obtained by randomly sampling $P$ rows/columns from the original $\mathcal{K}_{\ell}^k$, which takes $\mathcal{O}(P^2 nm)$ time. In practice, it randomly samples training examples, computes their kernel matrix and represents each training example by its similarity to the sampled examples and kernel matrix. It is well-known that this method can significantly accelerate our computations \citep{williams2001using,drineas2005approximating}, which is attained by using the approximate kernel matrix instead of the full one. Another advantage is that it is not necessary to compute or store the full kernel matrix, but only a submatrix requiring space complexity of $\mathcal{O}(Pmn)$. 

The Nystr\"{o}m method works effectively so long as the sampled training examples and the kernel matrix adequately represent the entire graph. However, if some training examples are distant from the randomly sampled ones, they may not be well-represented.

Another possible approximation method is the \textit{random feature approximation} \citep{rahimi2007random}, which can be executed in $\mathcal{O}(n m P \log F_\ell)$ time \citep{le2014fastfood}. Overall, using approximations significantly accelerates our algorithm's performance in terms of time and space complexity.

\section{PROOF OF THEOREM \ref{thm:generalization error main}}
\label{supp:proof of thm 1}

\begin{customthm}{\ref{thm:generalization error main}}
    \label{supp:thm:generalization error}
    Consider models with a single hidden layer for binary classification (i.e., $L=1,G=2$). Let $J(\cdot,\mathbf{y})$ be $M$-Lipchitz continuous for any fixed $\mathbf{y} \in [G]^n$ and let $\kappa_L$ be either the inverse polynomial kernel or the Gaussian RBF kernel (Recall Appendices \ref{sec:inverse}-\ref{sec:gauss}). For any valid activation function $\sigma$ (Recall Appendix \ref{sec:valid act}), there is a constant $C_{\sigma,L} (\mathcal{B}_L)$ such that with the radius $R_L := C_{\sigma,L} (\mathcal{B}_L) \sqrt{2F}$, the expected generalization error is at most:
    \begin{equation}
        \label{supp:eq:rademacher theory fin}
        \begin{aligned}
            \mathbb{E}_{\mathcal{X},\mathbf{y}} [J(\hat{\Psi}^{\hat{\mathcal{A}}_{L}}(\mathcal{X}), \mathbf{y})] \leq \inf_{\Phi \in \mathcal{F}_\text{gcn}} \mathbb{E}_{\mathcal{X},\mathbf{y}} [J(\Phi(\mathcal{X}), \mathbf{y})] + \frac{c M R_L \sqrt{\log(n(K+1)) \sum_{k=0}^K\mathbb{E}[||K_L^k(\mathcal{X})||_2]}}{\sqrt{m}}
        \end{aligned}
    \end{equation}
    where $c>0$ is a universal constant.
\end{customthm}

We begin by proving several properties that apply to the general multi-class case with models comprising of multiple hidden layers, and then we specialize them to single-hidden-layer models in the binary classification case. First, a wealthier function class shall be considered, which contains the class of GCNs. For each layer $\ell$, let $\mathcal{H}_\ell$ be the RKHS induced by the kernel function $\kappa_\ell$ and $||\cdot||_{\mathcal{H}_\ell}$ be the associated Hilbert norm. Recall that, given any matrix $\mathcal{Z} \in \mathbb{R}^{(K+1) \times F_{\ell-1}}$, we denote its $(k+1)^{\text{th}}$ row as $[\mathcal{Z}]_{k+1} := \mathbf{z}_k$ for any $k \in [K] \cup \{0\}$. Using this notation, consider the following function class:
\begin{equation}
    \label{eq:CGCN func}
    \begin{aligned}
        \mathcal{F}_\text{cgcn} := \{\Phi^{\mathcal{A}} = \phi^{{\mathcal{A}}_{L}} \circ \dots \circ \phi^{{\mathcal{A}}_{1}}: \phi^{{\mathcal{A}}_{\ell}} \in \mathcal{F}_{\ell,cgcn} \quad \forall \ell \in [L]\} \text{, where} \\
        \mathcal{F}_{\ell,cgcn} := \Bigg\{\phi^{{\mathcal{A}}_{\ell}}: \mathbb{R}^{n \times F_{\ell-1}^\star} \rightarrow \mathbb{R}^{n \times F_\ell} :  F_{\ell-1}^\star < \infty \text{ and } [\phi^{{\mathcal{A}}_{\ell}}(\mathcal{X})]_{if} = \tau_\ell^f(\mathcal{Z}_{i,\ell}(\mathcal{X})) \\
        \text{ where } \tau_\ell^f: \mathcal{Z} \in \mathbb{R}^{(K+1) \times F_{\ell-1}} \mapsto \sum_{k=0}^K \sum_{g=1}^{F_{\ell-1}^\star} [\mathbf{z}_k]_{g} [\mathbf{a}^f_{\ell k}]_{g} \text{ and } \\
        \|\tau_\ell^f\|_{\mathcal{H}_\ell} \leq C_{\sigma,\ell} (R_\ell) (K+1) \sqrt{F_\ell F_{\ell-1}} \Bigg\}
    \end{aligned}
\end{equation}
where $C_{\sigma,\ell} (R_\ell)$ depends solely on the chosen activation function $\sigma$ and the regularization parameter $R_\ell$. We further consider the function class of all graph filters at layer $\ell$ of a standard GCN:
\begin{equation}
    \label{eq:gcn layer ell}
    \mathcal{F}_{\ell, gcn} := \Bigg\{\Psi^{\mathcal{A}_\ell}: \Psi^{\mathcal{A}_{\ell}}(\mathcal{X}) := \sigma\Bigg(\sum_{k=0}^K \mathcal{S}^k \mathcal{X}_{\ell-1} \mathcal{A}_{\ell k}\Bigg) \text{ and } \max_{0 \leq k \leq K} \|\mathbf{a}^f_{\ell k}\|_2 \leq R_{\ell} \text{ } \forall f \Bigg\}
\end{equation}

Next, we show that the class $\mathcal{F}_\text{cgcn}$ contains the class of GCNs $\mathcal{F}_\text{gcn}$, making the former a reacher class of functions. Particularly, the class $\mathcal{F}_{\ell, cgcn}$ of kernelized filters at layer $\ell$ contains the class $\mathcal{F}_{\ell, gcn}$ of all graph filters at layer $\ell$ corresponding to a standard GCN.
\begin{lemma}
    \label{lemma:rich}
    For any valid activation function $\sigma$, there is some $C_{\sigma,\ell} (R_\ell)$ that solely depends on $\sigma$ and $R_\ell$ such that $\mathcal{F}_\text{gcn} \subset \mathcal{F}_\text{cgcn}$. In particular, $\mathcal{F}_{\ell, gcn} \subset \mathcal{F}_{\ell, cgcn}$.
\end{lemma}
\begin{proof}
    Recall that that any GCN $\Phi \in \mathcal{F}_\text{gcn}$ induces a nonlinear filter at each layer $\ell$ for the $f^{\text{th}}$ which can be characterized by $\tau_\ell^{f}: \mathcal{Z} \in \mathbb{R}^{(K+1) \times F_{\ell-1}} \mapsto \sigma(\sum_{k=0}^K \langle [\mathcal{Z}]_{k+1}, \mathbf{a}^f_{\ell k} \rangle)$, i.e., the GCN $\Phi$ is given by the following composition of functions $\Phi = \tau_L \circ \dots \circ \tau_1$, where we defined $\tau_\ell: \mathbf{z} \mapsto (\tau_\ell^{f}(\mathbf{z}))_{f \in [F_\ell]}$. Given any valid activation function, Corollary \ref{corollary:inverse-contains} and Corollary \ref{corollary:rbf} yield that $\tau_\ell^{f}$ resides within the reproducing kernel Hilbert space $\mathcal{H}_\ell$ and its Hilbert norm is upper bounded as $\|\tau_\ell^f\|_{\mathcal{H}_\ell} \leq C_\sigma(\|\sum_{k=0}^K \mathbf{a}^f_{\ell k}\|_2) \leq C_\sigma(R_\ell) \sqrt{K+1}$, which clearly satisfies the constraint in \eqref{eq:CGCN func} for any layer $\ell$, as desired.
\end{proof}

Recalling that $\hat{\Psi}^{\hat{\mathcal{A}}_{\ell}}(\cdot)$ is the predictor at each layer $\ell$ produced by Algorithm \ref{alg:two-layer}, we now formally prove that the predictor $\hat{\Psi}^{\hat{\mathcal{A}}_{L}}(\cdot)$ generated at the last layer is an empirical risk minimizer in the function class $\mathcal{F}_\text{cgcn}$.
\begin{lemma}
    \label{lemma:erm}
    By running Algorithm \ref{alg:two-layer} with the regularization parameter $\mathcal{B}_\ell = C_{\sigma,\ell} (R_\ell) \sqrt{F_\ell F_{\ell-1}}$ at layer $\ell$, the predictor $\hat{\Psi}^{\hat{\mathcal{A}}_{L}}(\cdot)$ generated at the last layer satisfies:
    \begin{equation}
        \label{eq:erm minimizer pred}
        \hat{\Psi}^{\hat{\mathcal{A}}_{L}} \in \argmin_{\Phi \in \mathcal{F}_\text{cgcn}} \frac{1}{m} \sum_{j=1}^m J(\Phi(\mathcal{X}^{(j)}), y^{(j)})
    \end{equation}
\end{lemma}
\begin{proof}
    Consider the function class of all CGCNs whose filters at layer $\ell$ have a nuclear norm of at most $\mathcal{B}_\ell$, given by:
    \begin{equation}
        \label{eq:CGCN bounded nuclear norm}
        \begin{aligned}
            \mathcal{F}_{\mathcal{B}_\ell} := \Bigg\{\Phi^{\mathcal{A}} = \Psi^{\mathcal{A}_L} \circ \dots \circ \Psi^{\mathcal{A}_1} \Bigg| & \Psi^{\mathcal{A}_\ell}: \mathcal{X} \mapsto \sum_{k=0}^K \mathcal{S}^k \mathcal{X}_{\ell-1} \mathcal{A}_{\ell k} \text{ and } \\ 
            & \mathcal{A}_{\ell k} \in \mathbb{R}^{F_{\ell-1} \times F_\ell} \text{ and } \mathcal{A}_{\ell}\|_* \leq \mathcal{B}_\ell \Bigg\}
        \end{aligned}
    \end{equation}
    Our goal is thus proving that $\mathcal{F}_{\mathcal{B}_\ell} \subset \mathcal{F}_\text{cgcn}$, and that any empirical risk minimizer in $\mathcal{F}_\text{cgcn}$ is also in $\mathcal{F}_{\mathcal{B}_\ell}$. We begin with showing that $\mathcal{F}_{\mathcal{B}_\ell} \subset \mathcal{F}_\text{cgcn}$. Let $\Phi^{\mathcal{A}} = \Psi^{\mathcal{A}_L} \circ \dots \circ \Psi^{\mathcal{A}_1} \in \mathcal{F}_{\mathcal{B}_\ell}$ such that $\Psi^{\mathcal{A}_\ell}(\mathcal{X}) = \sum_{k=0}^K \mathcal{S}^k \mathcal{X}_{\ell-1} \mathcal{A}_{\ell k}$ with $\|\mathcal{A}_{\ell}\|_* \leq \mathcal{B}_\ell$ for any $\ell \in [L]$. We shall prove that $\Psi^{\mathcal{A}_\ell} \in \mathcal{F}_{\ell,cgcn}$. Indeed, consider the singular value decomposition (SVD) of $\mathcal{A}_{\ell k}$, which we assume is given by $\mathcal{A}_{\ell k} = \sum_{s \in [F_{\ell-1}^\star]} \lambda_{\ell k s} \mathbf{w}_{\ell k s} \mathbf{u}_{\ell k s}^\top$ for some $F_{\ell-1}^\star < \infty$, where each left-singular vector $\mathbf{w}_{\ell k s}$ and each right-singular vector $\mathbf{u}_{\ell k s}$ are both unit column vectors, while each singular value $\lambda_{\ell k s}$ is a real number. Therefore, the filters corresponding to the $f^{\text{th}}$ feature can be written as: 
    \begin{equation}
        \label{eq:svd taps}
        \mathbf{a}^f_{\ell k} = \sum_{s \in [F_{\ell-1}^\star]} \lambda_{\ell k s} [\mathbf{w}_{\ell k s}]_f \cdot \mathbf{u}_{\ell k s}^\top = \sum_{s \in [F_{\ell-1}^\star]} \lambda_{\ell k s} w_{\ell k s}^f \cdot \mathbf{u}_{\ell k s}^\top
    \end{equation}
    As such, the value of the $f^{\text{th}}$ feature of node $i$ under $\Psi^{\mathcal{A}_\ell}$ can be rephrased as:
    \begin{equation}
        \label{eq:svd rephrase}
        \begin{aligned}
            [\Psi^{\mathcal{A}_{\ell}}(\mathcal{X})]_{if} = \sum_{k=0}^K \sum_{g=1}^{F_{\ell-1}} [\mathbf{z}_i^k(\mathcal{X}_{\ell-1})]_{g} [\mathbf{a}^f_{\ell k}]_{g} = \sum_{k=0}^K \Bigg\langle \mathbf{z}_i^k(\mathcal{X}_{\ell-1}), \sum_{s \in [F_{\ell-1}^\star]} \lambda_{\ell k s} w_{\ell k s}^f  \mathbf{u}_{\ell k s}^\top \Bigg\rangle
        \end{aligned}    
    \end{equation}
    Let $\mathbf{v}_{\ell}^k(\mathbf{z}) \in \mathbb{R}^{nm}$ be a vector whose $(i,j)$-th coordinate equals to $\kappa_{\ell}(\mathbf{z},\mathbf{z}_i^k(\mathcal{X}_{\ell-1}^{(j)}))$. Similarly to Appendix \ref{supp:kernelization}, denote $\tau_\ell^f(\mathbf{z}) := \sum_{k=0}^K \langle (Q_\ell^k)^\dag \mathbf{v}_{\ell}^k(\mathbf{z}), \sum_{s \in [F_{\ell-1}^\star]} \lambda_{\ell k s} w_{\ell k s}^f  \mathbf{u}_{\ell k s}^\top \rangle$. For any matrix $\mathcal{Z} \in \mathbb{R}^{(K+1) \times F_{\ell-1}}$ whose $(k+1)^{\text{th}}$ row is given by $[\mathcal{Z}]_{k+1} := \mathbf{z}_k$ for any $k \in [K] \cup \{0\}$, note that $\tau_\ell^f$ can be rephrased as $\tau_\ell^f(\mathbf{z}) := \sum_{k=0}^K \langle \mathbf{v}_{\ell}^k(\mathbf{z}), ((Q_\ell^k)^\top)^\dag \sum_{s \in [F_{\ell-1}^\star]} \lambda_{\ell k s} w_{\ell k s}^f \mathbf{u}_{\ell k s}^\top \rangle$. Thus, \eqref{eq:norm-constraint} implies that the Hilbert norm of the function $\tau_\ell^f$ satisfies:
    \begin{equation}
        \label{eq:}
        \begin{aligned}
            \|\tau_\ell^f\|_{\mathcal{H}_\ell} \leq \sum_{k=0}^K \Bigg\|(Q_\ell^k)^\top ((Q_\ell^k)^\top)^\dag \sum_{s \in [F_{\ell-1}^\star]} \lambda_{\ell k s} w_{\ell k s}^f \mathbf{u}_{\ell k s}^\top\Bigg\|_2 \leq \sum_{k=0}^K \Bigg\|\sum_{s \in [F_{\ell-1}^\star]} \lambda_{\ell k s} w_{\ell k s}^f \mathbf{u}_{\ell k s}^\top\Bigg\|_2  \\
            \leq \sum_{k=0}^K \sum_{s \in [F_{\ell-1}^\star]} |\lambda_{\ell k s}| \cdot |w_{\ell k s}^f| \cdot\|\mathbf{u}_{\ell k s}^\top\|_2 \\
            \leq \sum_{k=0}^K \sum_{s \in [F_{\ell-1}^\star]} |\lambda_{\ell k s}| = \sum_{k=0}^K \|\mathcal{A}_{\ell k}\|_* \leq C_{\sigma,\ell} (R_\ell) (K+1) \sqrt{F_\ell F_{\ell-1}}
        \end{aligned}
    \end{equation}    
    where we used the triangle inequality and the fact that the singular vectors are unit vectors. Since $[\Psi^{\mathcal{A}_{\ell}}(\mathcal{X})]_{if} = \tau_\ell^f(\mathcal{Z}_{i,\ell}(\mathcal{X}_{\ell-1}))$, we obtained that $\mathcal{F}_{\mathcal{B}_\ell} \subset \mathcal{F}_\text{cgcn}$, as desired.

    To conclude the proof, we next prove that any empirical risk minimizer $\Phi^{\mathcal{A}}$ in $\mathcal{F}_\text{cgcn}$ is also in $\mathcal{F}_{\mathcal{B}_\ell}$. By our proof in Appendix \ref{supp:kernelization}, the value of the $f^{\text{th}}$ feature of node $i$ at layer $\ell$ is then given by $\tau_\ell^f(\mathcal{Z}_{i,\ell}(\mathcal{X}_{\ell-1}^{(j)}))$, where, as specified in \eqref{eq:parameterization}, $\tau_\ell^f(\mathcal{Z}) = \sum_{k=0}^K \langle (Q_\ell^k)^\dag \mathbf{v}_{\ell}^k(\mathbf{z}_k), (Q_\ell^k)^\top \mathbf{\beta}_{\ell k}^f \rangle$ is a reparameterization of the graph filters at layer $\ell$ for some vector of the coefficients $\mathbf{\beta}_{\ell k}^f \in \mathbb{R}^{nm}$ whose $(i,j)$-th coordinate is $\beta_{\ell k,(i,j)}^f$. That is, $(Q_\ell^k)^\top \mathbf{\beta}_{\ell k}^f$ now act as the filters under this reparameterization. In Appendix \ref{supp:kernelization}, we have shown that the Hilbert norm of $\tau_\ell^f$ is then $\|\tau_\ell^f\|_{\mathcal{H}_\ell} = \|(Q_\ell^k)^\top \mathbf{\beta}_{\ell k}^f\|_2 \leq  C_{\sigma,\ell} (R_\ell)$, where the last inequality is by \eqref{eq:norm-constraint}. By \eqref{eq:nonlinear nuclear}, we further infer that the nuclear norm of the filters $(Q_\ell^k)^\top \mathbf{\beta}_{\ell k}^f$ in the new parameterization satisfies $\|(Q_\ell^k)^\top \mathbf{\beta}_{\ell k}^f\|_* \leq C_{\sigma,\ell} (R_\ell) \sqrt{F_\ell F_{\ell-1} (K+1)} \leq C_{\sigma,\ell} (R_\ell) (K+1) \sqrt{F_\ell F_{\ell-1}} $, meaning that $\Phi^{\mathcal{A}} \in \mathcal{F}_{\mathcal{B}_\ell}$, as desired.
\end{proof}


Next, we prove that $\mathcal{F}_{L, cgcn}$ is not "too big" by upper bounding its Rademacher complexity. The Rademacher complexity is a key concept in empirical process theory and can be used to bound the generalization error in our empirical risk minimization problem. Readers should refer to \citep{bartlett2002rademacher} for a brief overview on Rademacher complexity. The \textit{Rademacher complexity} of a function class $\mathcal{F} = \{\phi : \mathcal{D} \rightarrow\mathbb{R}\}$ with respect to $m$ i.i.d samples $\{\mathcal{X}^{(j)}\}_{j=1}^{m}$ is given by:
\begin{equation}
    \label{eq:Rademacher}
    \mathcal{R}_m(\mathcal{F}) := \mathbb{E}_{\mathcal{X},\varepsilon} \Bigg[\sup_{\phi \in \mathcal{F}} \frac{1}{m}\sum_{j=1}^{m}\varepsilon_j \phi(\mathcal{X}^{(j)}) \Bigg]
\end{equation}
where $\{\varepsilon_j\}_{j=1}^{m}$ are an i.i.d. sequence of uniform $\{\pm 1\}$-valued variables. Next, we provide the upper bound on the Rademacher complexity of the function class $\mathcal{F}_{L, cgcn}$ describing the graph filters. The lemma refers to the random kernel matrix $\mathcal{K}_{\ell}^k(\mathcal{X}) \in \mathbb{R}^{n \times n}$ induced by a graph signal $\mathcal{X} \in \mathbb{R}^{n \times F}$ drawn randomly from the population, i.e., the $(i,i')$-th entry of $\mathcal{K}_{\ell}^k(\mathcal{X})$ is $\kappa_\ell (\mathbf{z}_i^k(\mathcal{X}), \mathbf{z}_{i'}^{k}(\mathcal{X}))$. Furthermore, we consider the expectation $\mathbb{E}[||\mathcal{K}_{\ell}^k(\mathcal{X})||_2]$ of this matrix's spectral norm. While the previous lemmas consider the most general case where the number of features $F_L = G$ at the last layer may be any positive integer, observe that the following result focuses on the binary classification case where $F_L=1, G = 2$. As mentioned in Section \ref{sec:Convergence and Generalization Error}, Lemma \ref{lemma:Rademacher} can be readily extended to the multi-class case by employing a standard one-versus-all reduction to the binary case.

\begin{lemma}
    \label{lemma:Rademacher}
    In the binary classification case where $F_L = G = 2$, there exists a universal constant $c > 0$ such that:
    \begin{equation}
        \label{eq:Rademacher CGCN}
        \mathcal{R}_m(\mathcal{F}_{L, cgcn}) \leq \frac{c \cdot C_{\sigma,\ell} (\mathcal{B}_L)  \sqrt{2 F_{L-1} (K+1) \cdot \log(n(K+1)) \cdot \sum_{k=0}^K\mathbb{E}[||K_L^k(\mathcal{X})||_2]}}{\sqrt{m}}
    \end{equation}
\end{lemma}
\begin{proof}
    We begin with some preliminaries that apply to the general case and then leverage them in order to obtain the result. For brevity, we denote $R_\ell = C_{\sigma,\ell} (\mathcal{B}_\ell) \sqrt{F_\ell F_{\ell-1}}$. Consider some function $\Phi^{\mathcal{A}} = \phi^{{\mathcal{A}}_{L}} \circ \dots \circ \phi^{{\mathcal{A}}_{1}} \in \mathcal{F}_\text{cgcn}$. Without loss of generality, we assume that, each input input graph signal $\mathcal{X}^{(j)}$, the feature vector $[\mathcal{X}^{(j)}]_i = \mathbf{x}_i^{(j)}$ of each node $i$ resides in the unit $\ell_2$-ball (which can be obtain via normalization). Combined with $\kappa_\ell$'s being a continuous kernel with $\kappa_\ell(\mathbf{z}, \mathbf{z}) \leq 1$, Mercer’s theorem \citep[Theorem 4.49]{steinwart2008support} implies that there exists a feature map $\varphi_\ell : \mathbb{R}^{F_{\ell-1}} \rightarrow \ell^2(\mathbb{N})$ for which $\sum_{\xi=1}^\infty [\varphi_\ell(\mathbf{z})]_\xi [\varphi_\ell(\mathbf{z}')]_\xi = \sum_{\xi=1}^\infty \varphi_{\ell \xi}(\mathbf{z}) \varphi_{\ell \xi}(\mathbf{z}')$ converges uniformly and absolutely to $\kappa_{\ell}(\mathbf{z},\mathbf{z}')$. Accordingly, it holds that $\kappa_{\ell}(\mathbf{z},\mathbf{z}')=\langle \varphi_\ell(\mathbf{z}),\varphi_\ell(\mathbf{z}')\rangle$, which yields the following for each feature $f$ of any node $i$ at layer $\ell$ is $\tau_\ell^f(\mathcal{Z}_{i,\ell}(\mathcal{X}_{\ell-1}))$, where for any matrix $\mathcal{Z} \in \mathbb{R}^{(K+1) \times F_{\ell-1}}$ whose $(k+1)^{\text{th}}$ row is given by $[\mathcal{Z}]_{k+1} := \mathbf{z}_k$ for any $k \in [K] \cup \{0\}$:
    \begin{equation}
        \label{eq:feature map2}
        \tau_\ell^f(\mathcal{Z}) := \sum_{k=0}^K \langle \varphi_\ell(\mathbf{z}_k), \bar{\mathbf{a}}^f_{\ell k} \rangle
    \end{equation}
    where $\bar{\mathbf{a}}^f_{\ell k} \in \ell^2(\mathbb{N})$ is a countable-dimensional vector, due to the fact that $\varphi$ itself is a countable sequence of functions. Particularly, the Hilbert norm of $\tau_\ell^f$ is $\|\tau_\ell^f\|_{\mathcal{H}_\ell} = \|\bar{\mathbf{a}}^f_{\ell k}\|_2$. Now, for any graph signal $\mathcal{X}$, let $\Theta_i(\mathcal{X})$ be the linear operator that maps any sequence $\theta \in \ell^2(\mathbb{N})$ to the vector $[\langle \varphi_\ell(\mathbf{z}_i^0(\mathcal{X})), \theta \rangle, \dots,\langle \varphi_\ell(\mathbf{z}_i^K(\mathcal{X})), \theta \rangle]^\top \in \mathbb{R}^{K+1}$. Thereby, recalling that $\Phi^{\mathcal{A}} = \phi^{{\mathcal{A}}_{L}} \circ \dots \circ \phi^{{\mathcal{A}}_{1}}$ where $\phi^{{\mathcal{A}}_{\ell}} \in \mathcal{F}_{\ell,cgcn}$ satisfies $[\phi^{{\mathcal{A}}_{\ell}}(\mathcal{X}_{\ell-1})]_{if} = \tau_\ell^f(\mathcal{Z}_{i,\ell}(\mathcal{X}_{\ell-1}))$, then the $f^{\text{th}}$ feature of node $i$ at layer $\ell$ can be rephrased as:
    \begin{equation}
        \label{eq:feature map3}
        \begin{aligned}
            [\phi^{{\mathcal{A}}_{\ell}}(\mathcal{X}_{\ell-1})]_{if} = \mathbf{1}_{K+1}^{\top} \Theta_i(\mathcal{X}_{\ell-1}) \bar{\mathbf{a}}^f_{\ell k} = \tr \Big(\Theta_i(\mathcal{X}_{\ell-1}) (\bar{\mathbf{a}}^f_{\ell k} \mathbf{1}_{K+1}^{\top}) \Big) \\
            \Rightarrow \phi^{{\mathcal{A}}_{\ell}}(\mathcal{X}_{\ell-1}) = \Big[\tr \Big(\Theta_i(\mathcal{X}_{\ell-1}) (\bar{\mathbf{a}}^f_{\ell k} \mathbf{1}_{K+1}^{\top}) \Big)\Big]_{i \in [n], f \in [F_\ell]}
        \end{aligned}
    \end{equation}
    where $\mathbf{1}_{K+1}$ is the all-ones vector of dimension $K+1$. Note that the matrix $\bar{\mathcal{A}}_{\ell k} := \bar{\mathbf{a}}^f_{\ell k} \mathbf{1}_{K+1}^{\top}$ satisfies the following:
    \begin{equation}
        \label{eq:nuclear upper bound new}
        \|\bar{\mathcal{A}}_{\ell k}\|_* = \|\bar{\mathbf{a}}^f_{\ell k} \mathbf{1}_{K+1}^{\top}\|_* \leq \|\mathbf{1}_{K+1}^{\top}\|_2 \cdot \|\bar{\mathbf{a}}^f_{\ell k}\|_2 = \|\tau_\ell^f\|_{\mathcal{H}_\ell} \sqrt{K+1} \leq R_\ell \sqrt{K+1}
    \end{equation}

    Now, we obtain our upper bound for the binary classification case where $F_L = G = 2$. Combining \eqref{eq:feature map3} with \eqref{eq:nuclear upper bound new}, we obtain that the Rademacher complexity of $\mathcal{F}_{L, cgcn}$ is upper bounded as follows:
    \begin{equation}
        \label{eq:upper bound rademacher}
        \begin{aligned}
            \mathcal{R}_m(\mathcal{F}_{L, cgcn}) = \mathbb{E}\Bigg[\sup_{\phi^{{\mathcal{A}}_{L}} \in \mathcal{F}_{L, cgcn}} \frac{1}{m} \sum_{j=1}^{m}\varepsilon_j \Phi^{\mathcal{A}_L}(\mathcal{X}^{(j)}) \Bigg] = \\
            = \frac{1}{m} \mathbb{E} \Bigg[\sup_{\|\bar{\mathcal{A}}_{L k}\|_* \leq R_L \sqrt{K+1}} \tr \Big(\Big(\sum_{j=1}^{m} \varepsilon_j \Theta_i(\mathcal{X}_{L-1}^{(j)}) \Big) \bar{\mathcal{A}}_{L k}) \Big) \Bigg] \\
            \leq \frac{R_L \sqrt{K+1}}{m} \mathbb{E}\Bigg[\Big\|\sum_{j=1}^{m} \varepsilon_j \Theta_i(\mathcal{X}_{L-1}^{(j)}) \Big\|_2 \Bigg]
        \end{aligned}
    \end{equation}
    where the last equality uses H\"{o}lder’s inequality while using the duality between the nuclear norm and the spectral norm. Next, note that $\sum_{j=1}^{m} \varepsilon_j \Theta_i(\mathcal{X}_{L-1}^{(j)})$ can be thought of as a matrix with $n$ rows and infinitely many columns. We denote its sub-matrix comprising of the first $r$ columns as $\Theta_i^{(r)}(\mathcal{X}_{L-1}^{(j)})$ and let $\Theta_i^{(-r)}(\mathcal{X}_{L-1}^{(j)})$ be the remaining sub-matrix. Therefore:
    \begin{subequations}
        \label{eq:submat}
        \begin{align}
            \mathbb{E}\Bigg[\Big\|\sum_{j=1}^{m} \varepsilon_j \Theta_i(\mathcal{X}_{L-1}^{(j)}) \Big\|_2 \Bigg] \\
            \leq \mathbb{E}\Bigg[\Big\|\sum_{j=1}^{m} \varepsilon_j \Theta_i^{(r)}(\mathcal{X}_{L-1}^{(j)}) \Big\|_2 \Bigg] + \Bigg(\mathbb{E}\Bigg[\Big\|\sum_{j=1}^{m} \varepsilon_j \Theta_i^{(-r)}(\mathcal{X}_{L-1}^{(j)}) \Big\|_F^2 \Bigg]\Bigg)^{1/2}  \\
            \leq \mathbb{E}\Bigg[\Big\|\sum_{j=1}^{m} \varepsilon_j \Theta_i^{(r)}(\mathcal{X}_{L-1}^{(j)}) \Big\|_2 \Bigg] + \Bigg(nm \mathbb{E}\Bigg[\sum_{\xi=r+1}^{\infty} (\varphi_{\ell \xi} (\mathbf{z}))^2 \Bigg]\Bigg)^{1/2} \label{eq:submat 1}
        \end{align}
    \end{subequations}
    Since $\sum_{\xi=1}^{\infty} (\varphi_{\ell \xi} (\mathbf{z}))^2$ uniformly converges to $\kappa_\ell(\mathbf{z})$, the second term in \eqref{eq:submat 1} converges to $0$ as $r \rightarrow \infty$. Hence, it is sufficient to upper bound the first term in \eqref{eq:submat 1} and take the limit $r \rightarrow \infty$. By Bernstein inequality due to \cite[Theorem 2.1]{minsker2017some}, whenever $\tr(\Theta_i^{(r)}(\mathcal{X}_{L-1}^{(j)}) \Theta_i^{(r)}(\mathcal{X}_{L-1}^{(j)})^\top) \leq C_1$, there is a universal constant $c > 0$ such that the expected spectral norm is upper bounded by:
    \begin{equation}
        \label{eq:bern}
        \begin{aligned}
            \mathbb{E}\Bigg[\Big\|\sum_{j=1}^{m} \varepsilon_j \Theta_i^{(r)}(\mathcal{X}_{L-1}^{(j)}) \Big\|_2 \Bigg] \leq c \sqrt{\log(m C_1)} \mathbb{E}\Bigg[\Bigg(\sum_{j=1}^{m}\Big\|\Theta_i^{(r)}(\mathcal{X}_{L-1}^{(j)}) \Theta_i^{(r)}(\mathcal{X}_{L-1}^{(j)})^\top \Big\|_2\Bigg)^{1/2} \Bigg]  \\
            \leq c \sqrt{\log(m C_1)} (m \mathbb{E}[\|\Theta_i(\mathcal{X}) \Theta_i(\mathcal{X})^\top \|_2])^{1/2} \\
             \leq c \sqrt{\log(m C_1)} \Bigg(m \sum_{k=0}^K\mathbb{E}[||K_L^k(\mathcal{X})||_2]\Bigg)^{1/2} 
        \end{aligned}
    \end{equation}
    Notice that, due to the uniform kernel expansion $\kappa_{\ell}(\mathbf{z},\mathbf{z}') = \sum_{\xi=1}^\infty \varphi_{\ell \xi}(\mathbf{z}) \varphi_{\ell \xi}(\mathbf{z}')$, it holds that the trace norm is bounded as $\tr(\Theta_i^{(r)}(\mathcal{X}_{L-1}^{(j)}) \Theta_i^{(r)}(\mathcal{X}_{L-1}^{(j)})^\top) \leq K+1$ since $\kappa_\ell(\mathbf{z},\mathbf{z}') \leq 1$. Combining this with \eqref{eq:bern} and \eqref{eq:upper bound rademacher} implies the desired in \eqref{eq:Rademacher CGCN}.
\end{proof}

We are ready to prove the key lemma that will yield our main result. Recall that our current focus is on binary classification (i.e., $F_L = G = 2$). While our prior results consider models with multiple hidden layers, our main result regards the case of a single hidden layer $L = 1$. As mentioned in Section \ref{sec:Convergence and Generalization Error}, in the case of multiple hidden layers, our result applies to each individual layer separately. Using Lemmas \ref{lemma:rich}--\ref{lemma:Rademacher}, we are capable of comparing the CGCN predictor obtained by Algorithm \ref{alg:two-layer} against optimal model in the GCN class. Lemma \ref{lemma:erm} yields that the predictor $\hat{\Psi}^{\hat{\mathcal{A}}_{L}}(\cdot)$ generated at the last layer is an empirical risk minimizer in the function class $\mathcal{F}_\text{cgcn}$. As such, by the theory of Rademacher complexity (see, e.g., \citep{bartlett2002rademacher}), we obtain that:
\begin{equation}
    \label{eq:rademacher theory}
    \mathbb{E}_{\mathcal{X},\mathbf{y}} [J(\hat{\Psi}^{\hat{\mathcal{A}}_{L}}(\mathcal{X}), \mathbf{y})] \leq \inf_{\Phi \in \mathcal{F}_{L, cgcn}} \mathbb{E}_{\mathcal{X},\mathbf{y}} [J(\Phi(\mathcal{X}), \mathbf{y})] + 2M \cdot \mathcal{R}_m(\mathcal{F}_{L, cgcn}) + \frac{c}{\sqrt{m}}
\end{equation}
where $c$ is a universal constant a and we used the assumption that $J$ is $M$-Lipschitz. By Lemma \ref{lemma:rich}, we have that $\inf_{\Phi \in \mathcal{F}_{L, cgcn}} \mathbb{E}_{\mathcal{X},\mathbf{y}} [J(\Phi(\mathcal{X}), \mathbf{y})] \leq \inf_{\Phi \in \mathcal{F}_{L, gcn}} \mathbb{E}_{\mathcal{X},\mathbf{y}} [J(\Phi(\mathcal{X}), \mathbf{y})]$. Since $\mathcal{F}_\text{cgcn} = \mathcal{F}_{1,cgcn}$ due to $L=1$, this means that $\inf_{\Phi \in \mathcal{F}_\text{cgcn}} \mathbb{E}_{\mathcal{X},\mathbf{y}} [J(\Phi(\mathcal{X}), \mathbf{y})] \leq \inf_{\Phi \in \mathcal{F}_\text{gcn}} \mathbb{E}_{\mathcal{X},\mathbf{y}} [J(\Phi(\mathcal{X}), \mathbf{y})]$. Combined with \eqref{eq:rademacher theory}, we obtain the desired.

\begin{remark}
    \label{remark:activation functions}
    The constant $C_{\sigma,L} (\mathcal{B}_L)$ depends on the convergence rate of the polynomial expansion of the activation function $\sigma$ (See Appendix \ref{sec:choice-of-kernels}). Algorithmically, knowing the activation function is not required to run Algorithm \ref{alg:two-layer}. Theoretically, however, the choice of $\sigma$ matters for Theorem \ref{thm:generalization error main} when comparing the predictor produced by Algorithm \ref{alg:two-layer} against the best GCN, as it affects the representation power of such GCN. If a GCN with an activation function $\sigma$ performs well, CGCN will similarly generalize strongly. 
\end{remark}

\section{ADDITIONAL EXPERIMENTAL DETAILS}
\label{supp:Experimental Details}

\subsection{Dataset Descriptions}
\label{supp:Dataset Descriptions}
Table \ref{tab:dataset-stats} provides a summary of the statistics and characteristics of datasets used in this paper. All datasets are accessed using the TUDataset class \citep{morris2020tudataset} in PyTorch Geometric \citep{fey2019fast} (MIT License).

\begin{table}[h]
\centering
\caption{Statistics of the graph classification datasets used in our experiments.}
\begin{tabular}{lcccc}
\toprule
\textbf{Dataset} & \textbf{\#Graphs} &  \textbf{Avg. Nodes} & \textbf{Avg. Edges} & \textbf{\#Classes} \\
\midrule
MUTAG        & 118     & 18  & 20  & 2 \\
PTC-MR        & 344     & 26  & 51  & 2 \\
PROTEINS       & 1,113    & 39  & 73  & 2 \\
NCI1           & 4,110  & 30  & 32  & 2 \\
Mutagenicity   & 4,337  & 30  & 31  & 2 \\
\verb|ogbg-molhiv|	& 41,127 & 25 & 27 & 2 \\
\bottomrule
\end{tabular}
\label{tab:dataset-stats}
\end{table}

\subsection{Additional Implementation Details}
\label{supp:Implementation Details}

\textbf{Hardware and Software.} All experiments could fit on one GPU at a time. Most experiments were run on a server with 4 NVIDIA L40S GPUs. We ran all of our experiments in Python, using the PyTorch framework \citep{paszke2019pytorch} (\href{https://github.com/pytorch/pytorch/blob/main/LICENSE}{license URL}). We also make use of PyTorch Geometric (PyG) \citep{fey2019fast} (MIT License) for experiments with graph data.

\textbf{Implementation Details.} Particularly, all models were implemented using PyTorch Geometric (PyG) \citep{fey2019fast}. For GATv2 and GraphGPS, we use the official PyG implementations of their original papers, while for other baselines we use their PyG implementations, as standard in prior works (see, e.g., \citep{morris2019weisfeiler,rampavsek2022recipe}). In a high level, each neural architecture with $L\in\{2,6\}$ layers consists of an input encoder, stacked GNN layers with edge encoders, batch normalization as well as optional dropout and residual connections, followed by a global pooling step and a final output encoder (See Appendix \ref{sec:full implementation} for more details). The embedding dimension of each hidden layer is set to $32$, while non-convex models use a ReLU activation function between layers. For convex models, we use the Gaussian RBF kernel $\kappa(\mathbf{z},\mathbf{z}') := \exp(-\gamma \|\mathbf{z}-\mathbf{z}'\|_2^2)$ with $\gamma=0.2$, and compute an approximate kernel matrix via Nystr\"{o}m approximation \citep{williams2001using} with dimension $P=32$ (Recall Section 4.4 in the full paper). Further, since we do not focus on the impact of attention heads on performance, all convex and non-convex versions of GAT, GATv2, GraphTrans and GraphGPS use their default configurations of attention heads.

\textbf{Optimization.} All models are trained over $200$ epochs with the Adam optimizer \citep{kingma2014adam} using the cross entropy loss, a starting learning rate of $10^{-3}$, a minimum learning rate of $10^{-6}$ and a weight decay of $10^{-5}$. We use a ReduceLROnPlateau scheduler, which cuts the learning rate in half with a patience of $20$ epochs, and training ends when we reach the minimum learning rate. 

\subsection{Detailed Implementation Details of Neural Architectures}
\label{sec:full implementation}
\paragraph{Node Input Encoder.}
The first layer of each neural architecture consists of an input encoder implemented using an MLP, which transforms the raw data captured by the graph signal into a suitable representation for the GNN to process. Specifically:
\begin{itemize}
    \item If raw node features are available, then we an MLP whose input dimension is the number of features per node and its output dimension is the embedding dimension of $32$.
    \item If node features are absent, we use a learnable $\texttt{DiscreteEncoder}$ that maps integer node identifiers to vectors to the embedding dimension of $32$.
\end{itemize}

\paragraph{Edge Feature Encoders.}
We construct a separate edge encoder for each GNN layer. Each edge encoder maps raw edge features to the hidden dimension:
\begin{itemize}
    \item If edge features are provided, then we use a single-layer MLP.
    \item If edge features are absent, we use a $\texttt{DiscreteEncoder}$.
\end{itemize}

\paragraph{GNN Layers.}
We use a stack of $L$ GNN layers, each operating over the hidden feature dimension. Each GNN layer $\ell$ takes node features encoded using the input encoder, edge features encoded using the edge encoder corresponding to layer $\ell$, and an edge index matrix. Batch normalization is then applied to the output of the GNN layer, followed by a ReLU activation function. An optional dropout and a residual connection are then applied. Each layer thus has the following general form:
\begin{align*}
\text{GNN Layer} \rightarrow \text{BatchNorm} \rightarrow \text{ReLU} \rightarrow \text{Dropout (Optional)} \rightarrow\text{Residual Connection (Optional)}
\end{align*}


\paragraph{Global Pooling.}
After all GNN layers, we aggregate node features to obtain graph-level representations using \texttt{add} pooling to aggregate the learned node-level representations to a graph-level
embedding and a two-layer MLP for the final classification.

\paragraph{Output MLP.}
The final graph representation is passed through a two-layer MLP for the final classification. The first layer applies a linear projection followed by optional normalization and a nonlinearity. The second layer outputs logits or regression scores without an activation function.


\newpage
\section{Additional Experimental Results}
\label{supp:Additional Experiments}

\subsection{Heatmap of Classification Accuracy}
\label{supp:heatmap}

\begin{figure}[h!]
    \centering
    \includegraphics[width=\textwidth]{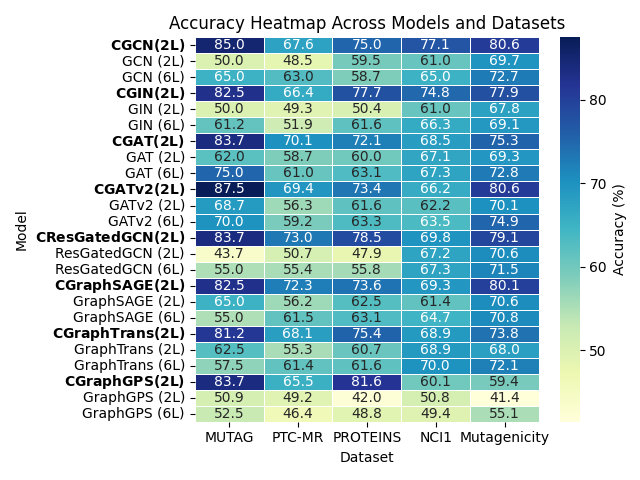}
    \caption{Heatmap of classification accuracy (\%) across all models and datasets. Rows correspond to models, columns to datasets. The names of convex models are \textbf{bolded} to emphasize their performance.}
    \label{fig:accuracy-heatmap}
\end{figure}
Figure~\ref{fig:accuracy-heatmap} provides a heatmap summarizing the accuracy of all models across the different datasets. It complements Table 1 in the full paper by visually reinforcing that convex models generally outperform their non-convex counterparts across nearly all datasets, with consistent and major improvements even for models like GATv2 and GraphGPS. This reaffirms the strength and generality of our convexification framework. For example, on MUTAG, the convex CGATv2 model reaches 87.5\% accuracy, outperforming the two-layer (68.7\%) and six-layer (70.0\%) GINs. Similarly, CGraphGPS achieves 83.7\%, exceeding the best non-convex variant by over 30\%. This trend holds across other datasets like PTC-MR, PROTEINS, and NCI1, where convex versions of CGraphSAGE, CGATv2, and CGraphGPS outperform non-convex baselines by 10--40\%.

\newpage
\subsection{Bar Plots}
\label{supp:bar plot}

\begin{figure}[h!]
    \centering
    \includegraphics[width=\linewidth]{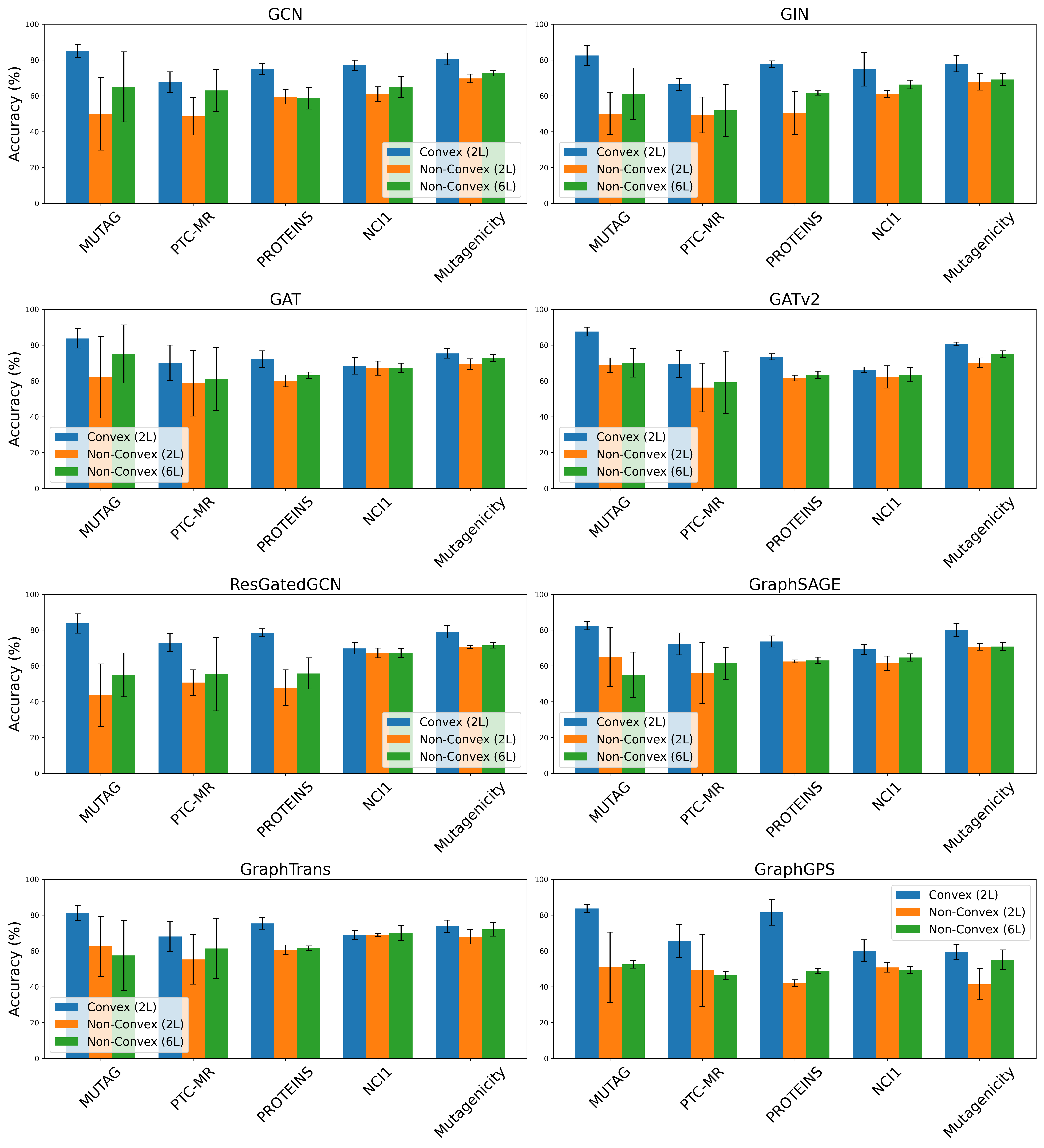}
    \caption{Accuracy comparison of two-layer convex GNNs (prefixed with ‘C’) against their standard two-layer and six-layer non-convex counterparts across five graph classification datasets. Bars represent the mean accuracy over four runs, and error bars denote the standard deviation. Datasets are ordered from left to right by increasing size.}
    \label{fig:accuracy bar plots}
\end{figure}
The bar plots in Figure \ref{fig:accuracy bar plots} highlight the consistent performance advantage of convex models over their non-convex counterparts across multiple datasets in a more straightforward way than numerical data alone. They further underscore how convex architectures provide not just dataset-specific improvements, but a general trend of robustness and high accuracy, even in the absence of large training sets. This emphasizes that our convex models harness stable optimization landscapes, enabling reliable learning where non-convex models exhibit erratic performance due to training instabilities.

\newpage
\subsection{Radar Plots}
\label{supp:radar}

\begin{figure}[h!]
    \centering
    \begin{tabular}{cc}
        \includegraphics[width=0.3\textwidth]{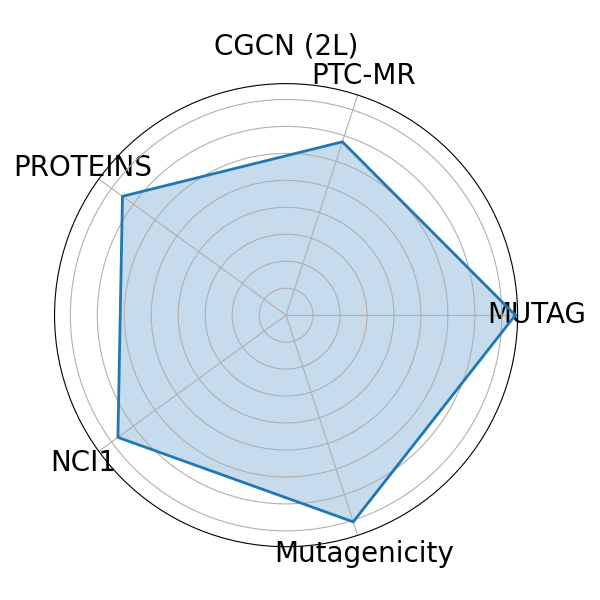} & 
        \includegraphics[width=0.3\textwidth]{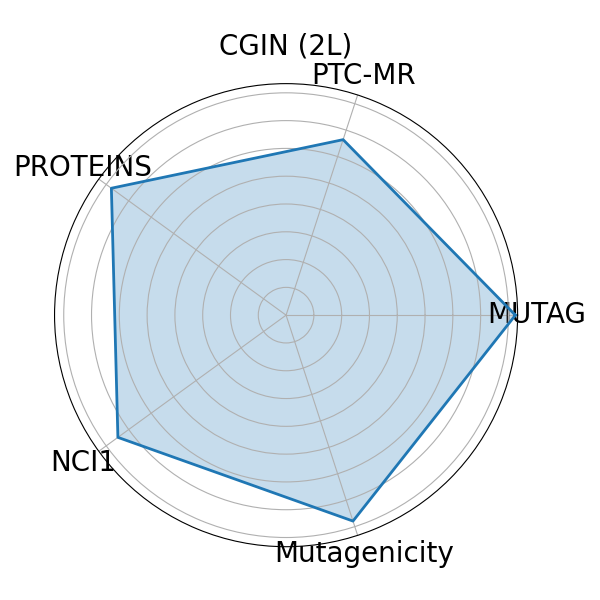} \\
        \includegraphics[width=0.3\textwidth]{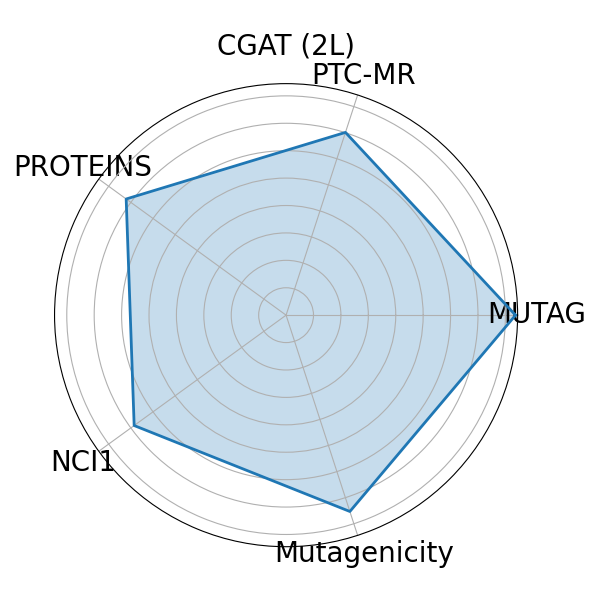} & 
        \includegraphics[width=0.3\textwidth]{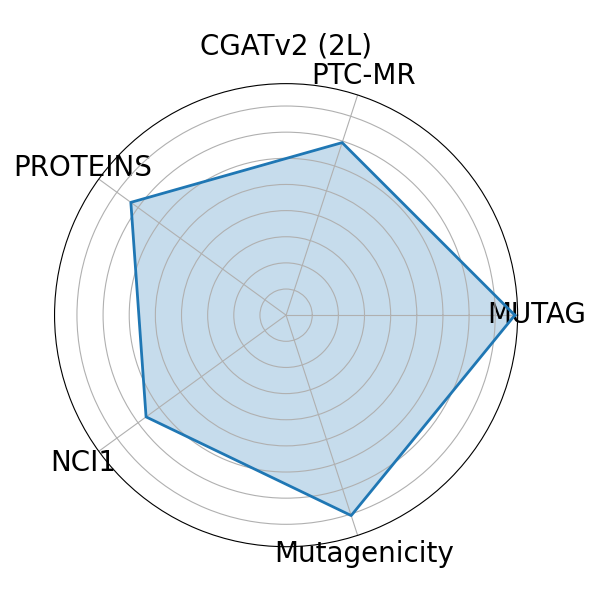} \\
        \includegraphics[width=0.3\textwidth]{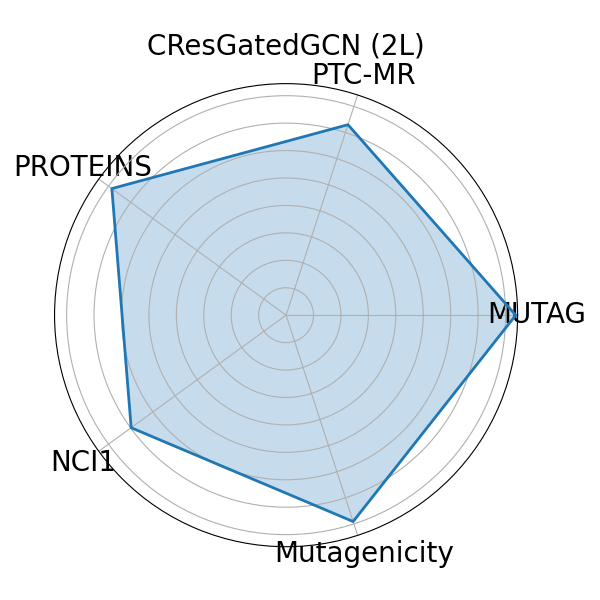} & 
        \includegraphics[width=0.3\textwidth]{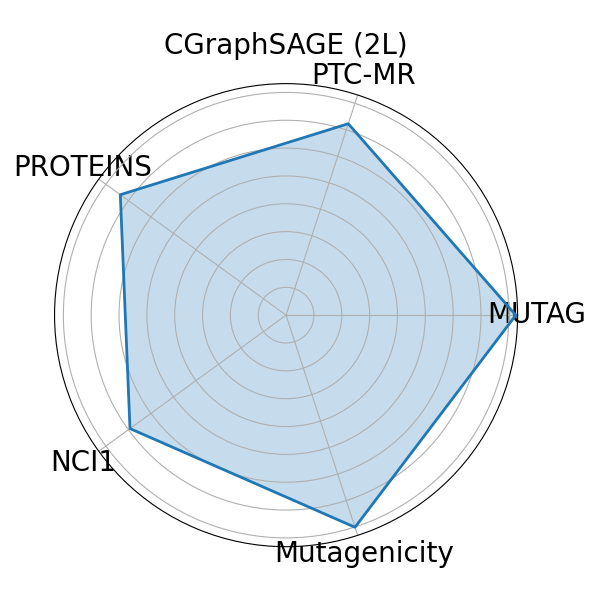} \\ 
        \includegraphics[width=0.3\textwidth]{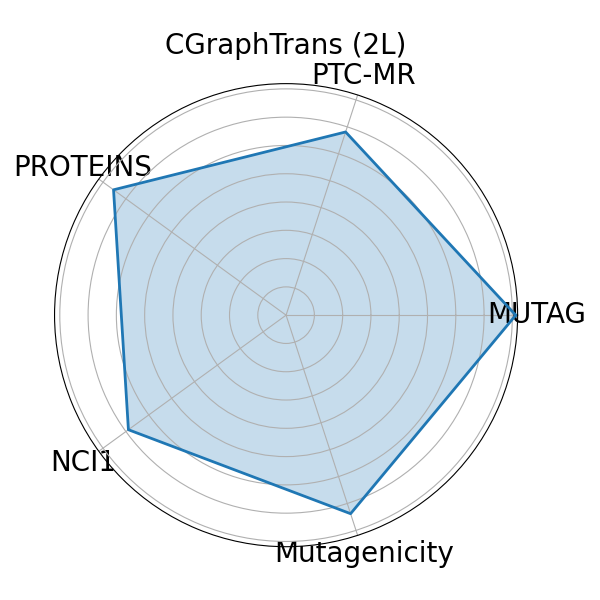} & \includegraphics[width=0.3\textwidth]{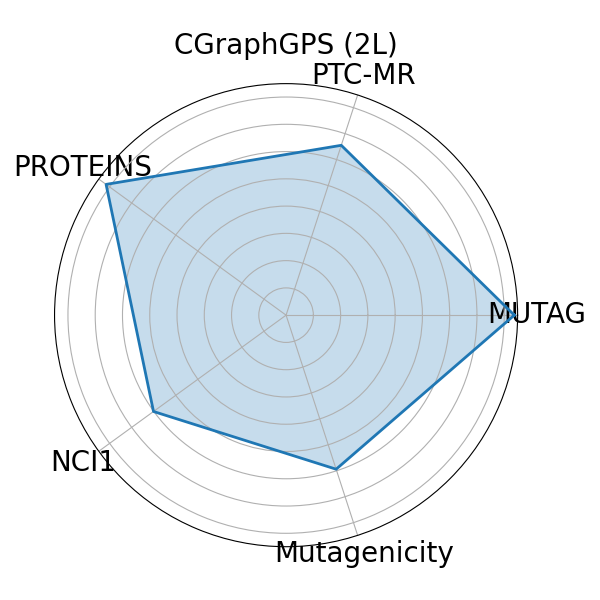} \\
    \end{tabular}
    \caption{Radar plots showing the performance of each two-layer convex model across five graph classification datasets. Each axis corresponds to a dataset, with higher values indicating better classification accuracy. These plots illustrate the consistency and robustness of convex models across diverse graph types.}
    \label{fig:radar-all}
\end{figure}

The radar plots in Figure \ref{fig:radar-all} offer a holistic view of how convex models perform across diverse datasets. They reveal that convex variants not only perform strongly on small datasets like MUTAG and PTC-MR, but also remain competitive or superior on larger datasets such as NCI1 and Mutagenicity. This multi-dimensional consistency further challenges the notion that expressive representations require deep, non-convex architectures, and instead supports that convex shallow models, when carefully designed, can adapt well across a spectrum of tasks and data scales.

\newpage
\subsection{Experimental Results for Deeper Convex Models}
\label{supp:Experimental Results for Deeper Convex Models}

Table \ref{tab:convex-vs-non-convex-mutagenicity-green} presents additional empirical results on Mutagenicity, comparing two-layer and six-layer convex versions of base models against their two-layer and six-layer non-convex counterparts\footnote{As demonstrated in Section \ref{sec:Experimental evaluation}, two-layer convex models already exhibit superior performance compared to their two- and six-layer non-convex counterparts when all models incorporate DropOut, batch normalization, node input encoders, and edge feature encoders. To evaluate the performance of deeper GNN models in isolation, in this section we exclude these additional components.}. Our results in Table \ref{tab:convex-vs-non-convex-mutagenicity-green} show that the accuracy of convex models can further improve with additional layers, which is consistent with our theoretical guarantees. Specifically, recall that Theorem \ref{thm:generalization error main} bounds the generalization error of Algorithm \ref{alg:two-layer} for single-hidden-layer models. As specified in Section \ref{sec:Convergence and Generalization Error}, our analysis extends layer-wise in the multi-layer setting, i.e., it applies to each hidden layer separately. The results in Table \ref{tab:convex-vs-non-convex-mutagenicity-green} therefore illustrate that errors remain controlled as depth increases. This offers strong empirical support for the practical effectiveness of multi-layer CGNNs, which is crucial since state-of-the-art GNN applications typically rely on deeper architectures.

\begin{table}[h!]
\centering
\caption{Accuracy ($\%\pm\sigma$) on Mutagenicity over $4$ runs for two-layer and six-layer convex versions of base models against their two-layer and six-layer non-convex counterparts. Convex variants start with 'C'. A two-layer model is marked by (2L), while a six-layer one is marked by (6L). Results of convex models surpassing their non-convex variants by a large gap ($\sim$10--40\%) are in \textbf{bold}, modest improvements are \underline{underlined}, and improvements of six-layer convex models over the corresponding two-layer convex models are in \textcolor{ForestGreen}{{dark green}}.}
\begin{tabular}{l c}
\toprule
\textbf{Model} & \textbf{Mutagenicity} \\
\midrule
\rowcolor{Gainsboro!50} CGCN (2L)       & $54.7\pm2.1$ \\
\rowcolor{Gainsboro!50} CGCN (6L)       & \textcolor{ForestGreen}{\textbf{$61.9\pm2.7$}} \\
GCN (2L)        & $54.7\pm1.1$ \\
GCN (6L)        & $54.7\pm1.9$ \\
\midrule
\rowcolor{Gainsboro!50} CGIN (2L)       & \underline{$56.8\pm1.8$} \\
\rowcolor{Gainsboro!50} CGIN (6L)       & \textcolor{ForestGreen}{\textbf{$59.1\pm1.8$}} \\
GIN (2L)        & $54.8\pm2.0$ \\
GIN (6L)        & $55.8\pm2.0$ \\
\midrule
\rowcolor{Gainsboro!50} CGAT (2L)       & \underline{$55.7\pm2.5$} \\
\rowcolor{Gainsboro!50} CGAT (6L)       & \textcolor{ForestGreen}{\textbf{$58.8\pm3.6$}} \\
GAT (2L)        & $54.8\pm2.0$ \\
GAT (6L)        & $54.8\pm2.0$ \\
\midrule
\rowcolor{Gainsboro!50} CGATv2 (2L)     & \underline{$56.6\pm2.7$} \\
\rowcolor{Gainsboro!50} CGATv2 (6L)     & \textcolor{ForestGreen}{\textbf{$62.3\pm4.2$}} \\
GATv2 (2L)      & $55.9\pm1.4$ \\
GATv2 (6L)      & $59.6\pm6.2$ \\
\midrule
\rowcolor{Gainsboro!50} CResGatedGCN (2L) & \underline{$56.2\pm2.7$} \\
\rowcolor{Gainsboro!50} CResGatedGCN (6L) & \textcolor{ForestGreen}{\textbf{$59.6\pm3.7$}} \\
ResGatedGCN (2L)  & $56.1\pm2.1$ \\
ResGatedGCN (6L)  & $56.1\pm2.1$ \\
\midrule
\rowcolor{Gainsboro!50} CGraphSAGE (2L)   & \underline{$55.1\pm2.4$} \\
\rowcolor{Gainsboro!50} CGraphSAGE (6L)   & \textcolor{ForestGreen}{\textbf{$66.3\pm5.3$}} \\
GraphSAGE (2L)    & $54.8\pm2.0$ \\
GraphSAGE (6L)    & $60.0\pm3.3$ \\
\midrule
\rowcolor{Gainsboro!50} CGraphTrans (2L)  & \underline{$55.5\pm2.3$} \\
\rowcolor{Gainsboro!50} CGraphTrans (6L)  & \textcolor{ForestGreen}{\textbf{$63.2\pm4.7$}} \\
GraphTrans (2L)   & $54.8\pm2.0$ \\
GraphTrans (6L)   & $59.9\pm4.6$ \\
\midrule
\rowcolor{Gainsboro!50} CGPS (2L)         & \underline{$57.7\pm4.5$} \\
\rowcolor{Gainsboro!50} CGPS (6L)         & \textcolor{ForestGreen}{\textbf{$59.4\pm4.1$}} \\
GPS (2L)          & $41.4\pm8.7$ \\
GPS (6L)          & $55.1\pm5.5$ \\
\bottomrule
\end{tabular}
\label{tab:convex-vs-non-convex-mutagenicity-green}
\end{table}

\newpage
\subsection{Experimental Results for DirGNNs}
\label{supp:Experimental Results for DirGNNs}
We herein provide additional empirical results for the recent baseline DirGNN \citep{rossi2024edge}, which is a generic wrapper for computing graph convolution on directed graphs. If the graph is originally undirected, applying DirGNN effectively treats it as directed. In our simulations, we extend GCN, GAT and GraphSAGE with DirGNN, obtaining DirGCN, DirGAT and DirGraphSAGE, respectively. In Table \ref{tab:convex-vs-non-convex-dirgnn}, we report results on the MUTAG and PROTEINS datasets, where the results are generated as described in Table \ref{tab:convex-vs-non-convex}. Particularly, convex architectures are highlighted in bold. Our two-layer convex models consistently outperform their non-convex two- and six-layer counterparts by 10-40\% accuracy.

\begin{table}[h!]
\centering
\caption{Accuracy ($\%\pm\sigma$) on MUTAG and PROTEINS over $4$ runs for two-layer convex versions of DirGNNs against their two-layer and six-layer non-convex counterparts. Convex variants start with 'C'. A two-layer model is marked by (2L), while a six-layer one is marked by (6L). Results of convex models surpassing their non-convex variants by a large gap ($\sim$10--40\%) are in \textbf{bold}, modest improvements are \underline{underlined}, and improvements of six-layer convex models over the corresponding two-layer convex models are in \textcolor{ForestGreen}{{dark green}}.}
\begin{tabular}{l c c}
\toprule
\textbf{Model} & \textbf{MUTAG} & \textbf{PROTEINS} \\
\midrule
\rowcolor{Gainsboro!50} CDirGCN (2L)   & \textcolor{ForestGreen}{$\mathbf{85.0\pm10.6}$} & \textcolor{ForestGreen}{$\mathbf{72.5\pm6.5}$} \\
DirGCN (2L)        & $50.0\pm10.0$ & $45.1\pm8.9$ \\
DirGCN (6L)        & $57.5\pm18.2$ & $47.0\pm10.6$ \\
\midrule
\rowcolor{Gainsboro!50} CDirGAT (2L)       & \textcolor{ForestGreen}{$\mathbf{75.0\pm6.1}$} & \textcolor{ForestGreen}{$\mathbf{73.6\pm4.3}$} \\
DirGAT (2L)        & $53.7\pm16.7$ & $59.7\pm5.8$ \\
DirGAT (6L)        & $61.2\pm24.5$ & $61.1\pm10.5$ \\
\midrule
\rowcolor{Gainsboro!50} CDirGraphSAGE (2L)       & \textcolor{ForestGreen}{$\mathbf{81.2\pm5.4}$} & \textcolor{ForestGreen}{$\mathbf{74.3\pm3.7}$} \\
DirGraphSAGE (2L)        & $60.0\pm14.5$ & $39.2\pm2.2$ \\
DirGraphSAGE (6L)        & $41.2\pm18.8$ & $51.5\pm9.4$ \\
\bottomrule
\end{tabular}
\label{tab:convex-vs-non-convex-dirgnn}
\end{table}

\subsection{Why do Convexified Message-Passing GNNs Excel?}
\label{supp:excel}

Next, we discuss why convexified GNNs perform well is valuable. We herein give a formal support for how each factor specified by the reviewer helps convexified models to excel, especially in small-data regimes.

\begin{enumerate}
    \item Constraining filters by a nuclear-norm ball limits the hypothesis class and yields favorable generalization gap bounds. Specifically, under the conditions of Theorem \ref{thm:generalization error main}, for each layer $\ell$ with graph filters $\mathcal{A}_{\ell}$ there exists a constant radius $B_{\ell}>0$ such that, if $\Vert\mathcal{A}_{\ell}\Vert_{\star} \leq B_{\ell}$, then the generalization gap scales as $\mathcal{O}(\frac{M C_\sigma(B_{\ell})}{\sqrt{m}})$ for a training set with $m$ samples and an $M$-Lipschitz continuous loss function, where $C_\sigma(B_{\ell})$ depends on the activation function $\sigma$. In contrast, unconstrained nonconvex training may effectively access much larger function classes (higher variance), which hurts generalization in small-data regimes. Our Theorem \ref{thm:generalization error main} thus indicates that nuclear-norm constraints control generalization for our CGNN class.

    \item Nuclear-norm projection acts as spectral smoothing, removing or attenuating directions with small singular values. Projecting the learned filter onto a nuclear-norm ball effectively suppresses small singular-value directions of the filter matrix. In particular, constraining nuclear norms encourages solutions that spread the singular-value mass in a way that uses fewer very large singular values, as nuclear norm sums them up, and penalizes many moderate ones. This eventually encourages low-rank solutions, similarly to trace-norm regularization. On graph-structured data, small singular-value directions often correspond to high-frequency components in the graph spectrum. In graph signal processing terms, low-rank approximations correspond to low-dimensional subspaces which omit high-frequency noise (see, e.g., \citep{ramakrishna2020user}). Thus, the nuclear-norm projection yields a smoother operator that suppresses noise and prevents overfitting to spurious high-frequency patterns. Namely, this explains the empirical gains we observe in low-data regimes.
\end{enumerate}

\end{document}